\documentclass[journal, onecolumn, draftcls]{IEEEtran}
\usepackage{amsmath,epsfig}
\usepackage{amsmath,graphicx}

\usepackage{amsmath,amssymb}
\usepackage{subfig,url, cite}
\usepackage{epstopdf}
\usepackage{float, cite}
\usepackage{amsfonts,amsthm}		
\usepackage{color}
\usepackage{caption}
\usepackage{algorithm}
\usepackage{algorithmic, bm, textcomp}
\usepackage[normalem]{ulem}

\newcommand{\bmath}[1]{\mathbf{\bm{#1}}}

\def \bx{\bmath{x}}
\def \ba{\bmath{a}}

\def \bA{\bmath{A}}

\def \bD{\bmath{D}}

\def \bI{\bmath{I}}
\def \bX{\bmath{X}}

\def \bZ{\bmath{Z}}
\def \bM{\bmath{M}}
\def \bX{\bmath{X}}

\def \bY{\bmath{Y}}

\def \1{{\mathbf 1}}

\def \E{{\mathbb E}}
\def \D{{\rm D}}
\def \H{{\rm H}}
\def \A{{\rm A}}
\def \pen{{\rm pen}}

\def \cX{{\cal X}}
\def \cS{{\cal S}}

\def \cD{{\cal D}}
\def \cA{{\cal A}}
\def \cL{{\cal L}}
\def \p{{ \gamma}}
\def \cA{{ C_{\A}}}
\def \cD{{ C_{\D}}}
\def \Xmax{{\rm X}_{\rm max}}
\def \Xmin{{\rm X}_{\rm min}}
\def \Amax{{\rm A}_{\rm max}}

\def \tr{ \rm tr}
 \def \lam{ \mathbf{\Lambda} }

\theoremstyle{plain}
\newtheorem{thmi}{Theorem}[section]

\newtheorem{cori}{Corollary}[section]

\newtheorem{lemmai}{Lemma}[section]

\newtheorem{remark}{Remark}[section]

\title{\huge Noisy Matrix Completion under Sparse Factor Models} 
\author{Akshay Soni, Swayambhoo Jain, Jarvis Haupt, and Stefano Gonella\thanks{Manuscript submitted October 31, 2014.  AS, SJ, and JH are with the Department of Electrical and Computer Engineering, and SG is with the Department of Civil, Environmental, and Geo- Engineering, at the University of Minnesota -- Twin Cities. Author emails: {\tt \{sonix022, jainx174, jdhaupt, sgonella\}@umn.edu}.  The first two (student) authors contributed equally to this work.  A preliminary version of this work will appear at the 2014 Global Conference on Signal and Information Processing (GlobalSIP).  JH graciously acknowledges support from NSF Award AST-1247885, and the DARPA Young Faculty Award, Grant N66001-14-1-4047.}}

\begin{document}
\maketitle
\begin{abstract}
This paper examines a general class of noisy matrix completion tasks where the goal is to estimate a matrix from observations obtained at a subset of its entries, each of which is subject to random noise or corruption.  Our specific focus is on settings where the matrix to be estimated is well-approximated by a product of two (a priori unknown) matrices, one of which is sparse.  Such structural models -- referred to here as ``sparse factor models'' -- have been widely used, for example, in subspace clustering applications, as well as in contemporary sparse modeling and dictionary learning tasks.  Our main theoretical contributions are estimation error bounds for sparsity-regularized maximum likelihood estimators for problems of this form, which are applicable to a number of different observation noise or corruption models.  Several specific implications are examined, including scenarios where observations are corrupted by additive Gaussian noise or additive heavier-tailed (Laplace) noise, Poisson-distributed observations, and highly-quantized (e.g., one-bit) observations.  We also propose a simple algorithmic approach based on the alternating direction method of multipliers for these tasks, and provide experimental evidence to support our error analyses.  
\end{abstract}
\begin{keywords}
Penalized maximum likelihood estimation, dictionary learning, matrix completion, subspace clustering.
\end{keywords}
\section{Introduction}
In recent years, there has been significant research activity aimed at the analysis and development of efficient \emph{matrix completion} methods, which seek to ``impute'' missing elements of a matrix given possibly noisy or corrupted observations collected at a subset of its locations.  Let $\bX^*\in\mathbb{R}^{n_1 \times n_2}$ denote a matrix whose elements we wish to estimate, and suppose that we observe $\bX^*$ at only a subset $\cS\subset[n_1]\times [n_2]$ of its locations, where $[n_1]=\{1,2,\dots,n_1\}$ is the set of all positive integers less or equal to $n_1$ (and similarly for $n_2$), obtaining at each $(i,j)\in \cS$ a noisy, corrupted, or inexact measurement denoted by $Y_{i,j}$.  The overall aim is to estimate $\bX^*$ given $\cS$ and the observations $\{Y_{i,j}\}_{(i,j)\in\cS}$.  Of course, such estimation problems may be ill-posed without further assumptions, since the values of $\bX^*$ at the unobserved locations could in general be arbitrary. A common approach is to augment the inference method with an assumption that the underlying matrix to be estimated exhibits some form of intrinsic \emph{low-dimensional} structure.  

One application where such techniques have been successfully utilized is collaborative filtering (e.g., as in the well-known \emph{Netflix Prize} competition \cite{Bell:07}).  There, the matrix to be estimated corresponds to an array of users' preferences or ratings for a collection of items (which could be quantized, e.g., to one of a number of levels); accurately inferring missing entries of the underlying matrix is a useful initial step in \emph{recommending} items (here, movies or shows) to users deemed likely to rate them favorably.  A popular approach to this problem utilizes a low-rank modeling assumption, which implicitly assumes that individual ratings depend on some unknown but nominally small number (say $r$) of features, so that each element of $\bX^*$ may be described as an inner product between two length-$r$ vectors -- one quantifying how well each of the features are embodied or represented by a given item, and the other describing a user's affinity for each of the features.  Recent works examining the efficacy of low-rank models for matrix completion include \cite{Candes:09:Exact, Candes:10:Power, Keshavan:10, Recht:11, Koren:09, Dai:10:SET, Ma:11}.

Several other applications where analogous ideas have been employed, but which leverage different structural modeling assumptions, include:

\begin{itemize}
\item \uline{\textbf{Sparse Coding for Image Inpainting and Demosaicing}:} 
Suppose that the underlying data to be estimated takes the form of an $n_1\times n_2$ color image, which may be interpreted as an $n_1\times n_2\times 3$ array (the three levels correspond to values in three color planes).  The image inpainting task amounts to estimating the image from a collection of (possibly noisy) observations obtained at individual pixel locations (so that at each pixel, either all or none of the color planes are observed), and the demosaicing task entails estimating the image from noisy measurements corresponding to only one of the $3$ possible color planes at each pixel.   The recent work \cite{Mairal:08:color} proposed estimation approaches for these tasks that leverage \emph{local shared structure} manifesting at the patch level. Specifically, in that work, the overall image to be estimated is viewed equivalently as a matrix comprised of vectorized versions of its small (e.g., $5\times 5\times 3$ or $8\times 8\times 3$) blocks, and the missing values are imputed using a structural assumption that this patch-based matrix be well-approximated by a product of two matrices, one of which is sparse.


\item \uline{\textbf{Sparse Models for Learning and Content Analytics}:} A recent work \cite{Lan:13} investigated a matrix completion approach to machine-based learning analytics.  There, the elements of the $n_1\times n_2$ matrix to be estimated, say $\bX^*$, are related to the probability with which one of $n_1$ questions will be answered correctly by one of $n_2$ ``learners'' through a \emph{link function} $\Phi: \mathbb{R}\rightarrow [0,1]$, so that the value $\Phi(X^*_{i,j})$ denotes the \emph{probability} with which question $i$ will be correctly answered by learner $j$.  The observed data are a collection of some $m< n_1 n_2$ binary values, which may be interpreted as (random) Bernoulli$(\Phi(X^*_{i,j}))$ variables.  The approach proposed in \cite{Lan:13} entails maximum-likelihood estimation of the unknown latent factors of $\bX^*$, under an assumption that $\bX^*$ be well-approximated by a sum of two matrices, the first being product of a sparse non-negative matrix (relating questions to some latent ``concepts'') and a matrix relating a learner's knowledge to the concepts, and the second quantifying the intrinsic difficulty of each question.

\item \uline{\textbf{Subspace Clustering from Missing Data}:} The general subspace clustering problem entails separating a collection of data points, using an assumption that similar points are described as points lying in the same subspace, so that the overall collection of data are represented as points belonging generally to a union of (ostensibly, low-dimensional) subspaces.  This general task finds application in image processing, computer vision, and disease detection, to name a few (see, e.g., \cite{Agrawal:98:Subspace, Tseng:00, Vidal:05:GPCA, Soltanolkotabi:12, Elhamifar:13:Sparse, Soltanolkotabi:13}, and the references therein).  One direct way to perform clustering in such applications entails approximating the underlying matrix $\bX^*$ whose columns comprise the (uncorrupted) data points by a product of two matrices, the second of which is sparse, so that the support (the set of locations of the nonzero elements) of each column of the sparse matrix factor identifies the subspace to which the corresponding column of $\bX^*$ belongs. 
\end{itemize}

While these examples all seem qualitatively similar in scope, their algorithmic and analytical tractability can vary significantly depending on the type of structural model adopted.  In the collaborative filtering application, for example, a desirable aspect of adopting low-rank models is that the associated inference (imputation) procedures can be relaxed to efficient convex methods that are amenable to precise performance analyses.  Indeed, the statistical performance of convex methods for low-rank matrix completion are now well-understood in noise-free settings (see, e.g., \cite{Candes:09:Exact, Candes:10:Power, Keshavan:10, Recht:11}), in settings where observations are corrupted by some form of additive uncertainty \cite{Keshavan:10:Noisy, Lee:10, Candes:10:MCNoise, Koltch:11, Negahban:12}, and even in settings where the observations may be interpreted as nonlinear (e.g., highly-quantized) functions of the underlying matrix entries \cite{Srebro:04, Davenport:12, Plan:14}.  In contrast, the aforementioned inference methods based on general bilinear (and sparse) factor models are difficult to solve to global optimality, and are instead replaced by tractable alternating minimization methods.  More fundamentally, the statistical performance of inference methods based on these more general bilinear models, in scenarios where the observations could arise from general (perhaps nonlinear) corruption models or could even be multi-modal in nature, has not (to our knowledge) been fully characterized.  

This work provides some initial results in this direction.  We establish a general-purpose estimation error guarantee for matrix completion problems characterized by any of a number of structural data models and observation noise/corruption models. For concreteness, we instantiate our main result here for the special case where the matrix to be estimated adheres to a \emph{sparse factor model}, meaning that it is well-approximated by the product of two matrices, one of which is sparse (or approximately so).  Sparse factor models are inherent in the modeling assumptions adopted in the aforementioned works on image denoising/demosaicing, content analytics, and subspace clustering, and are also at the heart of recent related efforts in dictionary learning \cite{Olshausen:97,Aharon:06, Mairal:09}.  Sparse factor models may also serve as a well-motivated extension to the low-rank models often utilized in collaborative filtering tasks.  There, while it is reasonable to assume that users' preferences will depend on a small number of abstract features, it may be that any particular user's preference relies heavily on only a subset of the features, and that the features that are most influential in forming a rating may vary from user to user.  Low rank models alone are insufficient for capturing this ``higher order'' structure on the latent factors, while this behavior may be well-described using the sparse factor models we consider here.  
 


\subsection{Our Contributions}

We address general problems of matrix completion under sparse factor modeling assumptions using the machinery of  \emph{complexity-regularized maximum likelihood} estimation.  Our main contributions come in the form of estimation error bounds that are applicable in settings where the available data correspond to an incomplete collection of noisy observations of elements of the underlying matrix (obtained at random locations), and under general (random) noise/corruption models.  We examine several specific implications of our main result, including for scenarios characterized by additive Gaussian noise or additive heavier-tailed (Laplace) noise, Poisson-distributed observations, and highly-quantized (e.g., one-bit) observations.  Where possible, we draw direct comparisons with existing results in the low-rank matrix completion literature, to illustrate the potential benefit of leveraging additional structure in the latent factors.  We also propose an efficient unified algorithmic approach based on the alternating direction method of multipliers \cite{Boyd:11} for obtaining a local solution to the (non-convex) optimizations prescribed by our analysis, and provide experimental evidence to support our error results.

\subsection{Connections with Existing Works}


As alluded above, our theoretical analyses here are based on the framework of complexity regularized maximum likelihood estimation \cite{Li:99:Thesis, Barron:99}, which has been utilized in a number of works to establish error bounds for Poisson estimation problems using multi scale models \cite{Kolaczyk:04, Willett:07}, transform domain sparsity models \cite{Raginsky:10}, and dictionary-based matrix factorization models \cite{Soni:14:Poisson}.  Here, our analysis extends that framework to the ``missing data'' scenarios inherent in matrix completion tasks (and also provides a missing-data extension of our own prior work on dictionary learning from 1-bit data \cite{Haupt:14:bit}).   

Our proposed algorithmic approach is based on the alternating direction method of multipliers (ADMM) \cite{Boyd:11}.  ADMM-based methods for related tasks in \emph{dictionary learning} (DL) were described recently in \cite{Rak:13}, and while our algorithmic approach here is qualitatively similar to that work, we consider missing data scenarios as well as more general loss functions that arise as negative log-likelihoods for our various probabilistic corruption models (thus generalizing these techniques beyond common squared error losses).  In addition, our algorithmic framework also allows for direct incorporation of constraints not only on estimates of the matrix factors, but also on the estimate of $\bX^*$ itself to account for entry-wise structural constraints that could arise naturally in many matrix completion scenarios.  Several other recent efforts in the DL literature have proposed algorithmic procedures for coping with missing data \cite{Xing:12:IncompleteHyperspectral,Zhou:12:NonparametricBayesIncomplete}, and a survey of algorithmic approaches to generalized low-rank modeling tasks is given in the recent work \cite{Udell:14}.  

Our inference tasks here essentially entail learning two factors in a bilinear model.  With a few notable exceptions (e.g., low-rank matrices, and certain non-negative matrices \cite{Donoho:03, Arora:12, Esser:12, Recht:12}), the joint non-convexity of these problems can complicate their analysis.  Recently, several efforts in the dictionary learning literature have established theoretical guarantees on identifiability, as well as local correctness of a number of factorization methods \cite{Aharon:06:Uniqueness,Gribonval:10:Identification,Geng:11,Spielman:13:Exact,Schnass:13:Identifiability,Agarwal:13}, including in noisy settings \cite{Jenatton:12}.  Our efforts here may be seen as a complement to those works, providing additional insight into the achievable \emph{statistical} performance of similar methods under somewhat general noise models.

The factor models we employ here essentially enforce that each column of $\bX^*$ lie in a union of linear subspaces. In this sense our efforts here are also closely related to problems in sparse principal component analysis \cite{Zou:06}, which seek to decompose the (sample) covariance matrix of a collection of data points as a sum of rank-one factors expressible as outer products of sparse vectors.  Several efforts have examined algorithmic approaches to the sparse PCA problem based on greedy methods \cite{Moghaddam:05} or convex relaxations \cite{dA:07, Shen:08, Zhang:12}, and very recently several efforts have examined the statistical performance of cardinality- (or $\ell_0$-) constrained methods for identifying the first sparse principal component \cite{Vu:12, Lounici:13}.  These latter approaches are related to our effort here, as our analysis below pertains to the performance of matrix completion methods utilizing an $\ell_0$ penalty on one of the matrix factors.   

Finally, we note that problems of subspace clustering from missing or noisy data have received considerable attention in recent years.  Algorithmic approaches to subspace clustering with missing data were proposed in \cite{Gruber:04, Vidal:08, Soltanolkotabi:13}, and several recent works have identified sufficient conditions under which tractable algorithms will provably recover the unknown subspaces in missing data (but noise-free) scenarios \cite{Eriksson:11:Highrank}.  Robustness of subspace clustering methods to missing data, additive noise, and potentially large-valued outliers were examined recently in \cite{Singh:12, Elhamifar:13:Sparse, Soltanolkotabi:13}.  

\subsection{Outline}

The remainder of this paper is organized as follows.  Following a brief discussion of several preliminaries (below), we formalize our problem in Section~\ref{sec:prob} and present our main result establishing estimation error guarantees for a general class of estimation problems characterized by incomplete and noisy observations. In Section \ref{sec:main} we discuss implications of this result for several specific noise models.  In Section \ref{sec:exp} we discuss a unified algorithmic approach to problems of this form, based on the alternating direction method of multipliers, and provide a brief experimental investigation that partially validates our theoretical analyses.  We conclude with a brief discussion in Section~\ref{sec:disc}.  Auxiliary material and detailed proofs are relegated to the appendix.

\subsection{Preliminaries}

To set the stage for the statement of our main result, we remind the reader of a few key concepts.  First, recall that for $p \leq 1$ a vector $\bx\in\mathbb{R}^n$ is said to belong to a weak-$\ell_p$ ball of radius $R>0$, denoted $\bx\in w\ell_p(R)$, if its ordered elements $|x_{(1)}| \geq |x_{(2)}| \geq \dots \geq |x_{(n)}|$ satisfy 
\begin{equation}
|x_{(i)}| \leq R i^{-1/p} \ \ \mbox{ for all } i\in \{1,2,\dots,n\},
\end{equation}  
see e.g., \cite{Baraniuk:10}. Vectors in weak-$\ell_p$ balls may be viewed as approximately sparse; indeed, it is well-known (and easy to show, using standard results for bounding sums by integrals) that for a vector $\bx\in w\ell_p(R)$, the $\ell_q$ error associated with approximating $\bx$ by its best $k$-term approximation obtained by retaining its $k$ largest entries in amplitude (denoted here by $\bx^{(k)}$) satisfies 
\begin{equation}
\|\bx-\bx^{(k)}\|_q \triangleq \left(\sum_{i=1}^n |x_i - x_i^{(k)}|^q\right)^{1/q} \leq R \ C_{p,q} \ k^{1/q-1/p},
\end{equation}
for any $q>p$, where $C_{p,q}$ is given by
\begin{equation}
C_{p,q} = \left(\frac{p}{q-p}\right)^{1/q}.
\end{equation}
For the special case $q\geq 2p$, we have $C_{p,q}\leq 1$, and so
\begin{equation}
\|\bx-\bx^{(k)}\|_q \leq R \ k^{1/q-1/p}.
\end{equation}

We also recall several information-theoretic preliminaries.  When $p(Y)$ and $q(Y)$ denote the pdf (or pmf) of a real-valued random variable $Y$, the Kullback-Leibler divergence (or KL divergence) of $q$ from $p$ is denoted $\D(p\|q)$ and given by
\begin{equation*}
\D(p\|q) = 
\E_{p}\left[\log\frac{p(Y)}{q(Y)}\right]
\end{equation*}
where the logarithm is taken to be the natural log.  By definition, $\D(p\|q)$ is finite only if the support of $p$ is contained in the support of $q$.  Further, the KL divergence satisfies $\D(p\|q) \geq 0$ and $\D(p\|q) = 0$ when $p(Y) = q(Y)$.   We also use the Hellinger affinity denoted by $\A(p,q)$ and given by
\begin{equation*}
\A(p,q) = \E_{p}\left[\sqrt{\frac{q(Y)}{p(Y)}}\right] = \E_{q}\left[\sqrt{\frac{p(Y)}{q(Y)}}\right]
\end{equation*}
Note that $\A(p,q)\geq 0$ essentially by definition, and a simple application of the Cauchy-Schwarz inequality gives that $\A(p,q) \leq 1$, implying overall that $0\leq \A(p,q)\leq 1$.  When $p$ and $q$ are parameterized by elements $X_{i,j}$ and $\widetilde{X}_{i,j}$ of matrices $\bX$ and $\widetilde{\bX}$, respectively, so that $p(Y_{i,j}) = p_{X_{i,j}}(Y_{i,j})$ and $q(Y_{i,j}) = q_{\widetilde{X}_{i,j}}(Y_{i,j})$, we use the shorthand notation $\D(p_{\bX}\|q_{\widetilde{\bX}}) \triangleq \sum_{i,j} \D(p_{X_{i,j}}\|q_{\widetilde{X}_{i,j}})$ and $\A(p_{\bX},q_{\widetilde{\bX}}) \triangleq \prod_{i,j} \A(p_{X_{i,j}},q_{\widetilde{X}_{i,j}})$.

Finally, for a matrix $\bM$ we denote by $\|\bM\|_0$ its number of nonzero elements, $\|\bM\|_1$ the sum of absolute values of its elements, $\|\bM\|_{\rm max}$ the magnitude of its largest element (in absolute value), and $\|\bM\|_{*}$ its nuclear norm (sum of singular values). 

\section{Problem Statement, Approach, and a General Recovery Result}\label{sec:prob}

As above, we let $\bX^*\in\mathbb{R}^{n_1 \times n_2}$ denote the unknown matrix whose entries we seek to estimate.   Our focus is on cases where the unknown matrix $\bX^*$ admits a factorization of the form
\begin{equation}\label{eqn:datamodel}
\bX^*= \bD^*\bA^*,
\end{equation}
where for some integer $r \leq n_2$, $\bD^*\in\mathbb{R}^{n_1\times r}$ and $\bA^*\in\mathbb{R}^{r\times n_2}$ are \emph{a priori unknown} factors. For pragmatic reasons, we assume that the elements of $\bD^*$, $\bA^*$, and $\bX^*$ are bounded, in the sense that
\begin{equation}
\|\bD^*\|_{\rm max}\leq 1, \ \ \|\bA^*\|_{\rm max}\leq \Amax, \ \ \mbox{ and } \ \ \|\bX^*\|_{\rm max}\leq \Xmax/2
\end{equation}
for some constants $0 < \Amax\leq (n_1\vee n_2) = \max\{n_1,n_2\}$ and $\Xmax\geq 1$.  Bounds on the amplitudes of the elements of the matrix to be estimated often arise naturally in practice\footnote{Here, the factor of $1/2$ in the bound on $\|\bX^*\|_{\rm max}$ is somewhat arbitrary -- any factor in $(0,1)$ would suffice -- and is chosen to facilitate our subsequent analysis.}, while our assumption that the entries of the factor matrices be bounded is essentially to fix scaling ambiguities associated with the bilinear model.  Our particular focus here will be on cases where (in addition to the entry-wise bounds) the matrix $\bA^*$ is \emph{sparse} (having no more than $k<rn_2$ nonzero elements), or \emph{approximately sparse}, in the sense that for some $p\leq 1$, all of its columns lie in a weak-$\ell_p$ ball of radius $\Amax$. 

Rather than acquire all of the elements of $\bX^*$ directly, we assume here that we only observe $\bX^*$ at a known \emph{subset} of its locations, obtaining for each observation a noisy or corrupted version of the underlying matrix entry.   Here, we will interpret the notion of ``noise'' somewhat generally in an effort to make our analysis amenable to any of a number of different corruption models; in what follows, we will model each entry-wise observation as a random quantity (either continuous or discrete-valued) whose probability density (or mass) function is parameterized by the true underlying matrix entry.  We denote by $\cS\subseteq [n_1]\times [n_2]$ the set of locations at which observations are collected, and assume that the sampling locations are random in the sense that for an integer $m$ satisfying $4\leq m \leq n_1n_2$ and $\p = m(n_1n_2)^{-1}$, $\cS$ is generated according to the independent Bernoulli($\p$) model so that each $(i,j)\in[n_1]\times [n_2]$ is included in $\cS$ independently with probability $\p$.   Then, given $\cS$, we model the collection of $|\cS|$ measurements of $\bX^*$ in terms of a collection $\{Y_{i,j}\}_{(i,j)\in\cS} \triangleq \bY_{\cS}$ of conditionally (on $\cS$) independent random quantities. Formally, we write the joint pdf (or pmf) of the observations as
\begin{equation}\label{eqn:obsmodel}
p_{\bX^*_{\cS}}(\bY_{\cS}) \triangleq \prod_{(i,j)\in\cS} p_{X^*_{i,j}}(Y_{i,j}),
\end{equation} 
where $p_{X^*_{i,j}}(Y_{i,j})$ denotes the corresponding scalar pdf (or pmf), and we use the shorthand $\bX^*_{\cS}$ to denote the collection of elements of $\bX^*$ indexed by $(i,j)\in\cS$.  In terms of this model, our task may be described concisely as follows:  given $\cS$ and corresponding noisy observations $\bY_{\cS}$ of $\bX^*$ distributed according to \eqref{eqn:obsmodel}, our goal is to estimate $\bX^*$ under the assumption that it admits a sparse factor model decomposition.   

Our approach will be to estimate $\bX^*$ via sparsity-penalized maximum likelihood methods; we consider estimates of the form
\begin{equation}\label{eqn:xhatsparse}
\widehat{\bX} = \arg\min_{\bX=\bD\bA \in \cX} \ \left\{-\log p_{\bX_{\cS}}(\bY_{\cS}) + \lambda \cdot \|\bA\|_{0} \right\},
\end{equation}
where $\lambda >0$ is a user-specified regularization parameter, $\bX_{\cS}$ is shorthand for the collection $\{X_{i,j}\}_{(i,j)\in\cS}$ of entries of $\bX$ indexed by $\cS$, and $\cX$ is an appropriately constructed class of candidate estimates.  To facilitate our analysis here, we take $\cX$ to be a countable class of estimates constructed as follows: first, for a specified $\beta\geq 1$, we set $L_{\rm lev} = 2^{\lceil \log_2 (n_1\vee n_2)^\beta \rceil}$ and construct ${\cal D}$ to be the set of all matrices $\bD\in\mathbb{R}^{n_1\times r}$ whose elements are discretized to one of  $L_{\rm lev}$ uniformly-spaced levels in the range $[-1, 1]$ and ${\cal A}$ to be the set of all matrices $\bA\in\mathbb{R}^{r\times n_2}$ whose elements either take the value zero, or are discretized to one of  $L_{\rm lev}$ uniformly-spaced levels in the range $[-\Amax, \Amax]$. Then, we let
\begin{equation}
\cX' \triangleq \left\{\bX = \bD\bA \ : \ \bD\in{\cal D}, \ \bA\in{\cal A}, \ \|\bX\|_{\rm max}\leq \Xmax\right\},
\end{equation}
and take $\cX$ to be any subset of $\cX'$. This general formulation will allow us to easily and directly handle additional constraints (e.g., non-negativity constraints on the elements of $\bX$, as arise in our treatment of the Poisson-distributed observation model), within the same unified analytical framework.

Our first main result establishes error bounds for sparse factor model matrix completion problems under general noise or corruption models, where the corruption is described by any generic likelihood model.  We state the result here as a theorem; its proof appears in Appendix~\ref{a:thmproof} and utilizes a key lemma that extends a main result of \cite{Li:99:Thesis} to ``missing data'' scenarios inherent in completion tasks.


\begin{thmi}\label{thm:main}
Let the sample set $\cS$ be drawn from the independent Bernoulli model with $\p = m(n_1n_2)^{-1}$ as described above, and let $\bY_{\cS}$ be described by \eqref{eqn:obsmodel}. If $\cD$ is any constant satisfying
\begin{equation}\label{eqn:thmcd}
\cD \geq \max_{\bX\in\cX} \max_{i,j}  \ D(p_{X^*_{i,j}}\|p_{X_{i,j}}),
\end{equation} 
where $\cX$ is as above for some $\beta \geq 1$, then for any
\begin{equation}
\lambda \geq 2 \cdot (\beta + 2) \cdot \left(1 + \frac{2\cD}{3}\right) \cdot \log(n_1 \vee n_2), 
\end{equation}
the complexity penalized maximum likelihood estimator \eqref{eqn:xhatsparse}
satisfies the (normalized, per-element) error bound
\begin{eqnarray}\label{eqn:thm1bnd}
\nonumber  \lefteqn{\frac{\E_{\cS,\bY_{\cS}}\left[-2\log \A(p_{\widehat{\bX}}, p_{\bX^*})\right]}{n_1 n_2} \leq  \frac{8 \cD \log m}{m}}\hspace{3em} &&\\
&+& 3 \cdot \min_{\bX\in\cX} \left\{ \frac{\D(p_{\bX^*}\|p_{\bX})}{n_1 n_2} + \left(\lambda + \frac{4\cD (\beta+2)\log (n_1 \vee n_2)}{3}\right)\left(\frac{n_1 p + \|\bA\|_0}{m}\right)\right\}.
\end{eqnarray}
\end{thmi} 

%

%

In the next section we consider several specific instances of this result, but we first note a few salient points about this result in its general form.  First, as alluded above, our result is not specific to any one observation model; thus, our general result will allow us to analyze the error performance of sparse factor matrix completion methods under a variety of different noise or corruption models.  Specialization to a given noise model requires us to only compute (or appropriately bound) the KL divergences and negative log Hellinger affinities of the corresponding probability densities or probability mass functions.  Second, our error bound is a kind of \emph{oracle} bound, in that it is specified in terms of a minimum over $\bX\in\cX$.  In practice, we may evaluate this oracle term for \emph{any} $\bX\in\cX$ and still obtain a valid upper bound (since our guarantee is in terms of the minimum).  In our analyses that follows we will impose assumptions on $\beta$ and $\bX^*$ that ensure $\bX^*$ be sufficiently ``close'' to some element $\bX$ of $\cX$. This will enable us to obtain non-trivial bounds on the first term in the oracle expression, and to subsequently quantify the corresponding normalized, per-element error (as described in terms of the corresponding negative log Hellinger affinity) by judiciously ``balancing'' the terms in the oracle expression.  This approach will be illustrated in the following section.

Finally, it is worth noting that the estimation strategies prescribed by our analysis are not computationally tractable.  Indeed, as written, formation of our estimators would require solving a combinatorial optimization, because of the $\ell_0$ penalty, as well as the optimization over the discrete set $\cX$.  However, it is worth noting that inference in the bilinear models we consider here is fundamentally challenging on account of the fact that these inference problems cannot directly be cast as (jointly) convex optimizations in the matrix factors.  In that sense, our results here may be interpreted as quantifying the performance of one (benchmark) estimation approach for sparse factor matrix completion under various corruption models.  (We discuss several extensions, including potential avenues for convexification, in Section~\ref{sec:disc}.)

\section{Implications for Specific Noise Models}\label{sec:main}

In this section we consider the implications of Theorem~\ref{thm:main} in four unique scenarios, characterized by additive Gaussian noise, additive heavier-tailed (Laplace) noise, Poisson-distributed observations, and quantized (one-bit) observations.  In each case, our aim is to identify the scaling behavior of the estimation error as a function of the key problem parameters. To that end, we consider for each case the fixed choice 
\begin{equation}\label{eqn:beta}
\beta = \max\left\{1, 1 + \frac{\log(8r\Amax/\Xmax)}{\log(n_1\vee n_2)}\right\}
\end{equation}
for describing the number of discretization levels in the elements of each of the matrix factors. Then, for each scenario (characterized by its own unique likelihood model) we consider a specific choice of $\cX$, and an estimate obtained according to \eqref{eqn:xhatsparse} with the specific choice
\begin{equation}\label{eqn:lamchoose}
\lambda = 2 \left(1 + \frac{2\cD}{3}\right) (\beta + 2) \cdot \log(n_1\vee n_2), 
\end{equation}
(where $\cD$ depends on the particular likelihood model), and simplify the resulting oracle bounds for both sparse and approximately sparse factors.  In what follows, we will make use of the fact that our assumption $\Xmax\geq 1$ implies $\beta = {\cal O}\left(\log(r\vee \Amax)/\log(n_1\vee n_2)\right)$, and so  
$(\beta + 2)\log(n_1\vee n_2) = {\cal O}\left(\log(n_1 \vee n_2)\right)$,
on account of the fact that $r<n_2$ and $\Amax < (n_1\vee n_2)$ by assumption.

\subsection{Additive Gaussian Noise}

We first examine the implications of Theorem~\ref{thm:main} in a setting where observations are corrupted by independent additive zero-mean Gaussian noise with known variance.  In this case, the observations $\bY_{\cS}$ are distributed according to a multivariate Gaussian density of dimension $|\cS|$ whose mean corresponds to the collection of matrix parameters at the sample locations, and with covariance matrix $\sigma^2 \bI_{|\cS|}$, where $\bI_{|\cS|}$ is the identity matrix of dimension $|\cS|$, so 
\begin{equation}\label{eqn:likGauss}
p_{\bX^*_{S}}(\bY_{S}) = \frac{1}{(2\pi\sigma^2)^{|S|/2}}\exp\left(-\frac{1}{2 \sigma^2} \ \|\bY_{S} - \bX^*_{S}\|_F^2\right),
\end{equation}
where we have used the representative shorthand notation $\|\bY_{S} - \bX^*_{S}\|_F^2 \triangleq \sum_{(i,j)\in S} (Y_{i,j}-X^*_{i,j})^2$.  In this setting  we have the following result; its proof appears in Appendix~\ref{a:gaussproof}.

\begin{cori}[Sparse Factor Matrix Completion with Gaussian Noise] \label{cor:Gauss}
Let $\beta$ be as in \eqref{eqn:beta}, let $\lambda$ be as in \eqref{eqn:lamchoose} with $\cD = 2\Xmax^2/\sigma^2$, and let $\cX = \cX'$. The estimate $\widehat{\bX}$ obtained via \eqref{eqn:xhatsparse} satisfies
\begin{equation} \label{eqn:gausssparse}
\frac{\E_{\cS,\bY_{\cS}}\left[\|\bX^*-\widehat{\bX}\|_F^2\right]}{n_1 n_2} = 
{\cal O}\left((\sigma^2 + \Xmax^2)\left(\frac{n_1 r + \|\bA^*\|_0}{m}\right)\log(n_1\vee n_2)\right)
\end{equation}
when $\bA^*$ is exactly sparse, having $\|\bA^*\|_0$ nonzero elements.  If, instead, the columns of $\bA^*$ are approximately sparse in the sense that for some $p\leq 1$ each belongs to a weak-$\ell_p$ ball of radius $\Amax$, then the estimate $\widehat{\bX}$ obtained via \eqref{eqn:xhatsparse}
satisfies
\begin{equation}\label{eqn:gaussweak}
\frac{\E_{\cS,\bY_{\cS}}\left[\|\bX^*-\widehat{\bX}\|_F^2\right]}{n_1 n_2} = {\cal O}\left(\Amax^2\left(\frac{n_2}{m}\right)^{\frac{2\alpha}{2\alpha+1}} +  (\sigma^2 + \Xmax^2) \left(\frac{n_1 r}{m} + \left(\frac{n_2}{m}\right)^{\frac{2\alpha}{2\alpha+1}}\right)\log(n_1\vee n_2)\right),
\end{equation}
where $\alpha=1/p-1/2$.
\end{cori}

\begin{remark} We utilize Big-Oh notation to suppress leading constants for clarity of exposition, and to illustrate the dependence of the bounds on the key problem parameters.  Our proofs for of each of the specific results provides the explicit constants.
\end{remark}

A few comments are in order regarding these error guarantees.   First, we note that our analysis provides some useful (and intuitive) understanding of how the estimation error decreases as a function of the number of measurements obtained, as well as the dimension and sparsity parameters associated with the matrix to be estimated.   Consider, for instance, the case when $\bA^*$ is sparse and where $\log m < n_1 r + \|\bA^*\|_0$ (which should often be the case, since $\log(m) \leq \log(n_1 n_2)$). In this setting,  our error bound shows that the dependence of the estimation error on the dimension ($n_1, n_2, r$) and sparsity ($\|\bA^*\|_0$) parameters, as well as the (nominal) number of measurements $m$ is  
\begin{equation}\label{eqn:ratio}
\frac{n_1 r + \|\bA^*\|_0}{m}  \ \log(n_1\vee n_2).
\end{equation}
We may interpret the quantity $n_1 r + \|\bA^*\|_0$ as the number of \emph{degrees of freedom} in the matrix $\bX^*$ to be estimated, and in this sense we see that  the error rate of the penalized maximum likelihood estimator exhibits characteristics of the well-known parametric rate (modulo the logarithmic factor).   Along related lines, note that in the case where columns of $\bA^*$ are approximately sparse, the $(n_2/m)^{\frac{2\alpha}{2\alpha+1}}$ term that arises in the error rate is reminiscent of error rates that arise when estimating approximately sparse vectors in noisy compressive sensing (e.g., see \cite{Candes:07:Dantzig, Haupt:06}).  Indeed, since $(n_2/m)^{\frac{2\alpha}{2\alpha+1}} \leq n_2 m^{-\frac{2\alpha}{2\alpha+1}}$, we see that the overall matrix estimation error may be interpreted as being comprised of errors associated with approximating the $n_2$ nearly-sparse columns of $\bA^*$ in this noisy setting, each of which would contribute a (normalized) error on the order of $m^{-\frac{2\alpha}{2\alpha+1}}$.  

Next, our error bounds provide some guidelines for identifying in which scenarios accurate estimation may be possible.  Consider a \emph{full sampling} scenario where the matrix $\bX^*=\bD^*\bA^*$ has a coefficient matrix with no more than $k$ nonzero elements per column (thus, $\|\bA^*\|_0 \leq n_2 k$).  Now, to ensure that 
\begin{equation}
\frac{n_1 r + \|\bA^*\|_0}{n_1 n_2}  \ \log(n_1\vee n_2) \preceq 1,
\end{equation}
(where the notation $\preceq$ suppresses leading constants) it is sufficient to have $n_1 n_2 \succeq n_1 r \log(n_1\vee n_2)$ and $n_1 n_2 \succeq \|\bA^*\|_0\log(n_1\vee n_2)$.  Simplifying a bit, we see that the first sufficient condition is satisfied when $n_2 \succeq 2r \log(n_1\vee n_2)$, or when the number of columns of the matrix $\bX^*$ exceeds (by a multiplicative constant and logarithmic factor) the number of columns of its dictionary factor $\bD^*$. Further, the second sufficient condition holds when $n_1 \succeq k\log(n_1\vee n_2)$, or when the number of measurements of each column exceeds (again, by a multiplicative constant and logarithmic factor) the number of nonzeros in the sparse representation of each column.  This latter condition is reminiscent of the sufficient conditions arising in sparse inference problems inherent in noisy compressive sensing (see, e.g., \cite{Candes:07:Dantzig, Needell:09}).  Analogous insights may be derived from our results for the subsampled regimes that comprise our main focus here (i.e., when $m < n_1n_2$). 

Further, we comment on the presence of the $\Xmax^2$ term present in the error bounds for both the sparse and nearly-sparse settings.  Readers familiar with the literature on matrix completion under low rank assumptions will recall that various forms of ``incoherence'' assumptions have been utilized to date as a means to ensure identifiability under various sampling models, and that the form of the resulting error bounds depend on the particular type of assumption employed.  For example, the authors of \cite{Candes:10:MCNoise} consider an additive noise model similar to here but employ incoherence assumptions that essentially enforce that the row and column spaces of the matrix to be estimated not be overly aligned with the canonical bases (reminiscent of initial works on noise-free matrix completion \cite{Candes:09:Exact}) and obtain estimation error bounds that do not depend on max-norm bounds of the matrix to be estimated (though the necessary conditions on the number of samples obtained do depend on the incoherence parameters).  The work \cite{Negahban:12} also examines matrix completion problems with additive noises but utilizes a different form of incoherence assumption formulated in terms of the ``spikiness'' of the matrix to be estimated (and quantified in terms of the ratio between the max norm and Frobenius norm).  There, the estimation approach entails optimization over a set of candidates that each satisfy a ``spikiness'' constraint, and the bounds so obtained scale in proportion to the max-norm of the matrix to be estimated (similar to here).  Incoherence assumptions manifesting as an assumed bound on the largest matrix element also arise in \cite{Davenport:12,Plan:14}.

One direct point of comparison to our result here is \cite{Koltch:11}, which considers matrix completion problems characterized by entry-wise observations obtained at locations chosen uniformly at random (with replacement), each of which may be modeled as corrupted by independent additive noise, and estimates obtained by nuclear norm penalized estimators. Casting the results of that work (specifically, \cite[Corollary 2]{Koltch:11}) to the setting we consider here, we observe that those results imply rank-$r$ matrices may be accurately estimated in the sense that 
\begin{eqnarray}
\frac{\|\bX^*-\widehat{\bX}\|_F^2}{n_1 n_2} &\leq& c \  
(\sigma \vee \Xmax)^2  \ \left(\frac{(n_1 \vee n_2) \ r}{m}\right) \ \log(n_1 + n_2)\\
&\leq& c' \  
(\sigma^2 + \Xmax^2) \ \left(\frac{(n_1 + n_2)  r}{m}\right) \log(n_1\vee n_2) 
\end{eqnarray}
with high probability, where $c,c'$ are positive constants.  Comparing this last result with our result \eqref{eqn:gausssparse}, we 
see that our guarantees exhibit the same effective scaling with the max-norm bound $\Xmax$, but can have an (perhaps significantly) improved error performance in the case where $\|\bA^*\|_0 \ll n_2 r$ -- precisely what we sought to identify by considering sparse factor models in our analyses.  The two bounds roughly coincide in the case where $\bA^*$ is not sparse, in which case we may take $\|\bA^*\|_0 = n_2 r$ in our error bounds.


\subsection{Additive Laplace Noise}

As another example, suppose that the observations $\bY_{\cS}$ are corrupted by independent additive heavier-tailed noises, each of which we model using a Laplace distribution with parameter $\tau>0$.  In this scenario, we have that
\begin{equation}\label{eqn:likLap}
p_{\bX^*_{S}}(\bY_{S}) =  \left(\frac{\tau}{2}\right)^{|S|} \exp\left(-\tau \ \|\bY_S - \bX^*_{S}\|_1\right),
\end{equation}
where we use $\|\bY_{S} - \bX^*_{S}\|_1 \triangleq \sum_{(i,j)\in S} |Y_{i,j}-X^*_{i,j}|$ for shorthand.  The following result holds; its proof appears in Appendix~\ref{a:lapproof}.

\begin{cori}[Sparse Factor Matrix Completion with Laplace Noise] \label{cor:Lap}
Let $\beta$ be as in \eqref{eqn:beta}, let $\lambda$ be as in \eqref{eqn:lamchoose} with $\cD = 2\tau \Xmax$, and let $\cX = \cX'$. The estimate $\widehat{\bX}$ obtained via \eqref{eqn:xhatsparse} satisfies
\begin{equation} \label{eqn:lapsparse}
\frac{\E_{\cS,\bY_{\cS}}\left[\|\bX^*-\widehat{\bX}\|_F^2\right]}{n_1 n_2} = 
{\cal O}\left(\left(\frac{1}{\tau} + \Xmax\right)^2\tau\Xmax \ \left(\frac{n_1 r + \|\bA^*\|_0}{m}\right)\log(n_1\vee n_2)\right),
\end{equation}
when $\bA^*$ is exactly sparse, having $\|\bA^*\|_0$ nonzero elements.  If, instead, for some $p\leq 1/2$ the columns of $\bA^*$ belong to a weak-$\ell_p$ ball of radius $\Amax$, then the estimate $\widehat{\bX}$ obtained via \eqref{eqn:xhatsparse}
satisfies
\begin{eqnarray}
\lefteqn{\frac{\E_{\cS,\bY_{\cS}}\left[\|\bX^*-\widehat{\bX}\|_F^2\right]}{n_1 n_2} =}&&\\
\nonumber && {\cal O}\left(\left(\frac{1}{\tau} + \Xmax\right)^2\tau\Amax \left(\frac{n_2}{m}\right)^{\frac{\alpha'}{\alpha'+1}} +  \left(\frac{1}{\tau} + \Xmax\right)^2\tau\Xmax \ \left(\frac{n_1 r}{m} + \left(\frac{n_2}{m}\right)^{\frac{\alpha'}{\alpha'+1}}\right) \log(n_1\vee n_2)\right),
\end{eqnarray}
where $\alpha'=1/p-1$.
\end{cori}

A few comments are in order regarding these results.  First, recall that our main theorem naturally provides error guarantees in terms of KL divergences and negative log Hellinger affinities.  However, here we state our bounds in terms of the average per element squared error, and draw comparisons with the previous case (and, perhaps, to make the results more amenable to interpretation).  To achieve this we employed a series of bounds -- quadratic (in the parameter difference) lower bounds on the negative log Hellinger affinities, and upper bounds on the KL divergences that are proportional to the absolute deviations between the parameters (see the proof for details).  It is interesting to note that this bounding approach, the error performance that we obtain for the case where $\bA^*$ is sparse again exhibits characteristics of the parametric rate, while we do obtain different error behavior as compared to the Gaussian noise case for the case where $\bA^*$ is nearly sparse.  As one specific example, consider the case where the coefficients of $\bA^*$ exhibit the ordered decay with $p=1/3$ (a parameter that is valid for both Corollaries~\ref{cor:Gauss} and \ref{cor:Lap}).  The error rate for the Gaussian noise setting in this case contains a term that decays on the order of $(m/n_2)^{-5/6}$, while here when the noise is heavy-tailed, the analogous term decays at a slower rate, like $(m/n_2)^{-2/3}$.   Overall, casting the error bounds all in terms of the same loss metric (here, $\ell_2$) makes our results directly amenable to such comparisons.

Along related lines, it is interesting to note that the estimation error bound here is slightly ``inflated'' relative to the Gaussian-noise counterparts (albeit with constants suppressed in each case).  Recall that the variance of a Laplace($\tau$) random variable is $2/\tau^2$; thus, the leading term $(1/\tau + \Xmax)^2 = {\cal O}(2/\tau^2 + \Xmax^2)$ here is somewhat analogous to the $(\sigma^2 + \Xmax^2)$ factor arising in the Gaussian-noise error bounds.  In this sense, we see that the factor of $\tau\Xmax$ in the Laplace-noise case appears to be ``extra.''  Here, this factor is effectively introduced by our attempt to cast the ``natural'' error guarantees arising from our analysis (which manifest in terms of negative log Hellinger affinities) into more interpretable squared-error bounds.  

\subsection{Poisson-distributed Observations}

We now consider an example motivated by applications where the observed data may correspond to discrete ``counts'' (e.g., in imaging applications). Suppose that the entries of the matrix $\bX^*$ are all non-negative and that our observation at each location $(i,j)\in\cS$ is a Poisson random variable with rate $X^*_{i,j}$.  In this setting, our matrix completion problem amounts to a kind of Poisson denoising task; we have that $\bY_{\cS}\in\mathbb{N}^{|\cS|}$ and  
\begin{equation}\label{eqn:likPoi}
p_{\bX^*_{S}}(\bY_{S}) = \prod_{(i,j)\in S} \frac{(X^*_{i,j})^{Y_{i,j}} e^{-X^*_{i,j}}}{(Y_{i,j})!}.
\end{equation}
In this case, we employ Theorem~\ref{thm:main} to obtain the following result; a sketch of the proof is provided in Appendix~\ref{a:poiproof}.

\begin{cori}[Sparse Factor Matrix Completion with Poisson Noise] \label{cor:Poi}
Suppose that the elements of the matrix $\bX^*$ to be estimated satisfy $\min_{i,j} |X^*_{i,j}|\geq \Xmin$ for some constant $\Xmin>0$.  Let $\beta$ be as in \eqref{eqn:beta}, let $\lambda$ be as in \eqref{eqn:lamchoose} with $\cD = 4\Xmax^2/\Xmin$, and let $\cX$ be the subset of $\cX'$ comprised of all candidate estimates having non-negative entries. The estimate $\widehat{\bX}$ obtained via \eqref{eqn:xhatsparse} satisfies
\begin{equation}
\frac{\E_{\cS,\bY_{\cS}}\left[\|\bX^*-\widehat{\bX}\|_F^2\right]}{n_1 n_2} = 
{\cal O}\left(\left(\Xmax + \frac{\Xmax}{\Xmin}\cdot\Xmax^2 \right) \left(\frac{n_1 r + \|\bA^*\|_0}{m}\right)\ \log(n_1\vee n_2)\right),
\end{equation}
when $\bA^*$ is exactly sparse, having $\|\bA^*\|_0$ nonzero elements.  If, instead, for some $p\leq 1$ the columns of $\bA^*$ belong to a weak-$\ell_p$ ball of radius $\Amax$, then the estimate $\widehat{\bX}$ obtained via \eqref{eqn:xhatsparse} satisfies
\begin{eqnarray}
\nonumber \lefteqn{\frac{\E_{\cS,\bY_{\cS}}\left[\|\bX^*-\widehat{\bX}\|_F^2\right]}{n_1 n_2} =}&&\\
&& {\cal O}\left(\Amax^2\left(\frac{\Xmax}{\Xmin}\right) \left(\frac{n_2}{m}\right)^{\frac{2\alpha}{2\alpha+1}} +\left(\Xmax + \frac{\Xmax}{\Xmin}\cdot\Xmax^2 \right)  \left(\frac{n_1 r}{m} + \left(\frac{n_2}{m}\right)^{\frac{2\alpha}{2\alpha+1}}\right) \log(n_1\vee n_2)\right),
\end{eqnarray}
where $\alpha=1/p-1/2$.
\end{cori}

As in the previous case, our analysis approach here entails bounding (appropriately) the KL divergence and negative log Hellinger affinities each in terms of squared Frobenius norms; similar bounding methods were employed in \cite{Raginsky:10}, which analyzed a compressive sensing sparse vector reconstruction task under a Poisson observation model.   Overall, we observe an interesting behavior relative to the preceding two cases.  Recall that the bounds for the setting where $\bA^*$ is exactly sparse, for each of the previous two cases, exhibited a leading factor that was essentially the sum of the variance and $\Xmax^2$.  In each of those cases, the per-observation noise variances were independent of the underlying matrix entry; in contrast, Poisson-distributed observations exhibit a variance equal to the underlying rate parameter.  So, in this sense, we might interpret the $(\Xmax + (\Xmax/\Xmin)\Xmax^2)$ term as roughly corresponding to a ``worst-case'' variance plus $\Xmax^2$.  Indeed, when $\Xmax/\Xmin$ is upper-bounded by a (small) constant; then, this leading factor is ${\cal O}(\Xmax + \Xmax^2)$, somewhat analogously to the leading factor arising in the Laplace-noise and Gaussian-noise bounds.  More generally, that the error behavior in Poisson denoising tasks be similar to the Gaussian case is perhaps not surprising. Indeed, a widely used approach in Poisson inference tasks is to employ a \emph{variance stabilizing transformations}, such as the Anscombe transform \cite{Anscombe:48}, so that the transformed data distribution be ``approximately'' Gaussian.  

It is worth commenting a bit further on our \emph{minimum} rate assumption on the elements of $\bX^*$, that each be no smaller than some constant $\Xmin>0$. Similar assumptions were employed in \cite{Raginsky:10}, as well as other works that examine Poisson denoising tasks using the penalized ML analysis framework (e.g., \cite{Kolaczyk:04}).  Here, this $\Xmin$ parameter shows up in the denominator of a leading factor in our bound, suggesting that the bounds become more loose as the estimation task transitions closer to scenarios characterized by ``low-rate'' Poisson sources.  Indeed, closer inspection of our error bounds as stated above shows that they diverge (tend to $+\infty$) as $\Xmin$ tends to zero, suggesting that the estimation task becomes more difficult in ``low rate'' settings.  Contrast this with classical analyses of scalar Poisson rate estimation problems show that the Cramer-Rao lower bound associated with estimating the rate parameter $\theta$ of a Poisson random variable using $n$ iid observations is $\theta/n$, and this error is achievable with the sample average estimator.  This suggests that the estimation problem actually becomes \emph{easier} as the rate decreases, at least for the scalar estimation problem.  On this note, we briefly mention several recent works that rectify this apparent discrepancy, for matrix estimation tasks as here \cite{Soni:14:Poisson} , and for sparse vector estimation from Poisson-distributed compressive observations \cite{Jiang:14}.  The ideas underlying those works might have applicability for the completion problems we consider here, but this extension may require imposing different (or even much stronger) forms of incoherence assumptions on the matrix to be estimated as compared to the bounded-entry condition we adopt here.  We do not pursue those extensions here, opting instead to state our result as a direct instantiation of our main result in Theorem~\ref{thm:main}.

\subsection{Quantized (One-bit) Observation Models}

We may also utilzie our main result to assess the estimation performance in scenarios where entry-wise observations of the matrix are quantized to few bits, or even a \emph{single bit}, each.  Such quantized observations are natural in collaborative filtering applications such as the aforementioned \emph{Netflix} problem, where users' ratings are quantized to fixed levels.  One may also envision applications in distributed estimation tasks where one seeks to estimate some underlying matrix from highly-quantized observations of a subset of its entries; here, the quantization could serve as a mechanism for enforcing global communication rate constraints (e.g., when the data is transmitted to a centralized location for inference).  Our general framework would facilitate analysis of observations quantized to any of a number of levels; here, for concreteness, we consider a one-bit observation model.  

Formally, given a sampling set $\cS$ we suppose that our observations are conditionally (on $\cS$) independent random variables described by
\begin{equation}\label{eqn:obs}
Y_{i,j} = \1_{\{Z_{i,j} \geq 0\}}, \ \ (i,j)\in\cS,
\end{equation}
where    
\begin{equation}
Z_{i,j} = X^*_{i,j} - W_{i,j}, 
\end{equation}
the $\{W_{i,j}\}_{i\in[m], j\in[n]}$ are some iid continuous zero-mean real scalar ``noises'' having probability density function and cumulative distribution function $f(w)$ and $F(w)$, respectively, for $w\in\mathbb{R}$, and $\1_{\{{\cal E}\}}$ denotes the indicator of the event ${\cal E}$ that takes the value $1$ when ${\cal E}$ occurs and zero otherwise.   Note that in this model, we assume that the individual noise realizations $\{W_{i,j}\}_{(i,j)\in\cS}$ are unknown (but we assume that the noise \emph{distribution} is known).  Stated another way, we may interpret the observations modeled as above essentially as quantized noisy versions of the true matrix parameters (the minus sign on the $W_{i,j}$'s is merely a modeling convenience here, and is intended to simplify the exposition).  Under this model, it is easy to see that each $Y_{i,j}$ is a Bernoulli random variable whose parameter is related to the true parameter through the cumulative distribution function.  Specifically,  note that for any fixed $(i,j)\in\cS$, we have that $\Pr(Y_{i,j} = 1) = \Pr(W_{i,j}\leq X^*_{i,j}) = F(X^*_{i,j})$.  Thus, in this scenario, we have that $\bY_{\cS}\in\{0,1\}^{|\cS|}$ and
\begin{equation}\label{eqn:likBern}
p_{\bX^*_{\cS}}(\bY_{\cS}) = \prod_{(i,j)\in\cS} \ \left[F(X_{i,j}^*)\right]^{Y_{i,j}} \ \left[1-F(X_{i,j}^*)\right]^{1-Y_{i,j}}
\end{equation}
We will also assume here that $\Xmax$ and $F(\cdot)$ are such that $F(\Xmax) < 1$ and $F(-\Xmax)>0$; it follows that the true Bernoulli parameters (as well as the Bernoulli parameters associated with candidate estimates $\bX\in\cX$) are bounded away from $0$ and $1$; these assumptions will allow us to avoid some pathological scenarios in our analysis.

Given the above model and assumptions, we may establish the following result; the proof is provided in Appendix~\ref{a:bernproof}.

\begin{cori}[Sparse Factor Matrix Completion from One-bit Observations] \label{cor:Bern}

\sloppypar Let $\beta$ be as in \eqref{eqn:beta},  let $\cX = \cX'$, and let $p_{\bX^*_{\cS}}$ be of the form in \eqref{eqn:likBern} with $F(\Xmax) < 1$ and $F(-\Xmax)>0$. Define
\begin{equation}
c_{F,\Xmax} \triangleq \left(\sup_{|t|\leq \Xmax} \frac{1}{F(t)(1-F(t))}\right)\cdot \left(\sup_{|t|\leq \Xmax} f^2(t)\right),
\end{equation}
and
\begin{equation}
c'_{F,\Xmax} \triangleq \inf_{|t|\leq \Xmax} \frac{f^2(t)}{F(t)(1-F(t))},
\end{equation}
and let $\lambda$ be as in \eqref{eqn:lamchoose} with $\cD = 2 c_{F,\Xmax} \Xmax^2$. The estimate $\widehat{\bX}$ obtained via \eqref{eqn:xhatsparse} satisfies
\begin{equation}\label{eqn:bernsparse}
\frac{\E_{\cS,\bY_{\cS}}\left[\|\bX^*-\widehat{\bX}\|_F^2\right]}{n_1 n_2} = 
{\cal O}\left( \left(\frac{c_{F,\Xmax}}{c'_{F,\Xmax}}\right)\left(\frac{1}{c_{F,\Xmax}} + \Xmax^2\right)\ \left(\frac{n_1 r + \|\bA^*\|_0}{m}\right) \log(n_1\vee n_2)\right),
\end{equation}
when $\bA^*$ is exactly sparse, having $\|\bA^*\|_0$ nonzero elements.  If, instead, for any $p\leq 1$ the columns of $\bA^*$ belong to a weak-$\ell_p$ ball of radius $\Amax$, then the estimate $\widehat{\bX}$ obtained via \eqref{eqn:xhatsparse} satisfies
\begin{eqnarray}\label{eqn:bernweak}
\lefteqn{\frac{\E_{\cS,\bY_{\cS}}\left[\|\bX^*-\widehat{\bX}\|_F^2\right]}{n_1 n_2} =}&&\\ 
\nonumber  &&{\cal O}\left( \left(\frac{c_{F,\Xmax}}{c'_{F,\Xmax}}\right)\Amax^2 \left(\frac{n_2}{m}\right)^{\frac{2\alpha}{2\alpha+1}} +  \left(\frac{c_{F,\Xmax}}{c'_{F,\Xmax}}\right)\left(\frac{1}{c_{F,\Xmax}} + \Xmax^2\right)\ \left(\frac{n_1 r}{m} + \left(\frac{n_2}{m}\right)^{\frac{2\alpha}{2\alpha+1}}\right)\log(n_1\vee n_2)\right).
\end{eqnarray}
where $\alpha=1/p-1/2$.
\end{cori}

It is interesting to compare the results of \eqref{eqn:bernsparse} and \eqref{eqn:bernweak} with the analogous results \eqref{eqn:gausssparse} and \eqref{eqn:gaussweak}.  Specifically, we see that our estimation error guarantees for each case exhibit the same fundamental dependence on the dimension, sparsity, and (nominal) number of measurements, with the primary difference overall arising in the form of the leading factors (that in the one-bit case depend on the specific distribution of the $W_{i,j}$ terms).  That the estimation errors for rate-constrained tasks approximately mimic that of their Gaussian-corrupted counterparts was observed in earlier works on rate-constrained parameter estimation (see, e.g., \cite{Luo:05, Ribiero:06}), and more recently in \cite{Davenport:12}, which considered low-rank matrix completion from one-bit measurements, using a generative model analogous to the model we consider here.

It is also worth noting that the cdf $F(\cdot)$ that we specify here could be replaced by any of a number of commonly-used \emph{link functions}. For example, choosing $F(x) = \int_{-\infty}^{x} \frac{1}{\sqrt{2\pi}} e^{-t^2/2} dt$ to be the cdf of a standard Gaussian random variable gives rise to the well-known probit model, while taking $F(x)$ to be the logistic function, $F(x) = \frac{1}{1+e^{-x}}$, leads to the logit regression model.  In this sense, our results are related to classical methods on inference in generalized linear models (see, e.g., \cite{McCullagh:89}); a key distinction here is that we assume both of the factors in the bilinear form to be unknown.  

Finally, we briefly compare our results with the results of \cite{Davenport:12} for low-rank matrix completion from one-bit observations. In that work, the authors consider maximum-likelihood optimizations over a (convex) set of max-norm and nuclear-norm constrained matrices, and show that the estimates so-obtained satisfy
\begin{equation}
\frac{\|\bX^*-\widehat{\bX}\|_F^2}{n_1 n_2} = {\cal O}\left( C_{F,\Xmax} \Xmax \sqrt{\frac{(n_1 + n_2) r}{m}}\right) 
\end{equation}
with high probability, where $C_{F,\Xmax}$ is a parameter that depends on the max-norm constraint cdf $F(\cdot)$ and pdf $f(\cdot)$, somewhat analogously to the leading factor of $(c_{F,\Xmax}/c'_{F,\Xmax})$ in our bounds.  It is interesting to note a main qualitative difference between that result and ours.  For concreteness, let us consider the case where $\bA^*$ is not sparse, so that we may set $\|\bA^*\|_0 = n_2r$ in \eqref{eqn:bernsparse}.  In that case, it is easy to see that the overall estimation error behavior predicted by our bound \eqref{eqn:bernsparse} scales in proportion to ratio between the number of degrees of freedom ($(n_1 + n_2) r $) and the nominal number of measurements $m$, while the bound in \cite{Davenport:12} scales according to the square root of that ratio.  The authors of \cite{Davenport:12} proceed to show that the estimation error rate they obtain is minimax optimal over their set of candidate estimates; on the other hand, our bound appears (at least up to leading factors) to be tighter for this case where $\bX^*$ is exactly low-rank (e.g., setting $\|\bA^*\|_0 = n_2r$ in our bound) and $m\geq c (n_1 + n_2)r$ for a constant $c>1$.  That said, our approach also enjoys the benefit of having the rank or an upper bound for it be known (and being combinatorial in nature!), while the procedure in \cite{Davenport:12} assumes only a bound on the nuclear norm of the unknown matrix.  Whether our bounds here exhibit minimax-optimal estimation error rates for matrix completion under sparse factor models and for the several various likelihood models we consider here is still an open question since (to our knowledge) lower bounds for these problems have not yet been established (but are a topic of our ongoing efforts).

\section{Experimental Evaluation}\label{sec:exp}

In this section we provide experimental evidence to validate the error rates established by our theoretical results.  Recall (as noted above) that the original problem \eqref{eqn:xhatsparse} we aim to solve has multiple sources of non-convexity, including the bilinear matrix factor model (i.e., $\bX = \bD\bA$), the presence of the $\ell_{0}$ penalty, and the discretized sets $\cal D$ and $\cal A$. In what follows, we undertake a slight relaxation of \eqref{eqn:xhatsparse} replacing the sets $\cal D$ and $\cal A$ by their convex hulls (and with slight overloading of notation in what follows, we refer to these new sets also as $\cal D$ and $\cal A$).  With this relaxation, the set $\cal X$ becomes a set of all matrices $\bX \in \mathbb{R}^{n_{1} \times n_{2}}$ with bounded entries. Note that with these simplifying relaxations, the sets $\cal X$, $\cal D$ and $\cal A$ are convex. 

Now, for each likelihood model, our aim is to solve a constrained maximum likelihood problem of the form
\begin{eqnarray} \label{eq:mainOpt}
\min_{\bD\in\mathbb{R}^{n_1\times r},\bA\in\mathbb{R}^{r\times n_2} } && \sum_{i,j} s_{i,j} \ell (Y_{i,j}, X_{i,j}) + I_{\cal X}(\bX) + I_{\cal D}(\bD) +  I_{\cal A}(\bA) + \lambda \|\bA\|_0\\
\nonumber \mbox{s.t. } &&  \bX = \bD\bA .
\end{eqnarray}
 where  $\ell(Y_{i,j}, X_{i,j}) = -\log(p_{X_{i,j}}(Y_{i,j}))$ is the negative log-likelihood for the corresponding noise model, $s_{i,j}$ is a selector taking the value $1$ when $(i,j) \in \cS$ and $0$ otherwise, $\lambda \ge 0$ is a regularization parameter, and each of $I_{\cal X}(\cdot),~ I_{\cal D}(\cdot),$ and $I_{\cal A}(\cdot) $  are the indicator functions of the sets $\cal X,~ \cal D$ and $\cal A$ respectively\footnote{Recall that the indicator function is defined as a function that takes values $0$ or $\infty$ depending on whether its argument is an element of the set described as the subscript.}. Here, we have that each of the indicator functions is separable in the individual entries of its argument, e.g. $I_{\cal X}(\bX) = \sum_{i,j}I_{{\cal X}_{i,j}}(X_{i,j})$, and similarly for the indicator functions of ${\cal D}$ and ${\cal A}$.  

We propose a solution approach based on the Alternating Direction Method of Multipliers (ADMM) \cite{Boyd:11}. First we write the augmented Lagrangian of \eqref{eq:mainOpt} as 
\begin{equation}\label{eqn:Aug_lag}
\cL(\bD,\bA,\bX,\lam ) = \sum_{i,j} s_{i,j}\ell (Y_{i,j}, X_{i,j}) +  I_{\cal X }(\bX) + I_{\cal D}(\bD) +  I_{\cal A}(\bA) + \lambda \|\bA\|_0 + \tr\left(\lam (\bX - \bD\bA  )  \right) + \frac{\rho}{2} \|\bX - \bD\bA \|_F^2,
\end{equation}
where $\lam$ is a matrix of Lagrange multiplier parameters and $\rho > 0$ is a parameter.  Then, starting with some feasible $\bA^{(0)}, \bD^{(0)}, \lam^{(0)}$ we iteratively update $\bX$, $\bA$, $\bD$, and $\lam$ according to
\begin{eqnarray}\label{eqn:ADMM_Steps}
\mathbf{(S1:)} \ \bX^{(k+1)} &:=&    \arg \min_{\bX\in\mathbb{R}^{n_1\times n_2}} \cL(\bD^{(k)},\bA^{(k)},\bX,\lam^{(k)} )   \\
\mathbf{(S2:)} \ \bA^{(k+1)} &:=& \arg \min_{\bA\in\mathbb{R}^{r\times n_2}} \cL(\bD^{(k)},\bA,\bX^{(k+1)},\lam^{(k)} )   \\
\mathbf{(S3:)} \ \bD^{(k+1)} &:=&  \arg \min_{\bD\in\mathbb{R}^{n_1\times r}} \cL(\bD,\bA^{(k+1)},\bX^{(k+1)},\lam^{(k)} )   \\
\mathbf{(S4:)} \ \lam^{(k+1)} &=& \lam^{(k)} + \rho (\bX^{(k+1)} - \bD^{(k+1)}\bA^{(k+1)}),
\end{eqnarray}
until convergence, which here is quantified in terms of when norms of primal and dual residuals become sufficiently small (along the lines of the criteria described in \cite{Boyd:11}). Next we describe how to solve each of these steps.  

Solving $\mathbf{S1}$ involves the following optimization problem
\begin{eqnarray}
 \min_{\bX\in\mathbb{R}^{n_1\times n_2}} && \sum_{i,j} s_{i,j} \ell (Y_{i,j}, X_{i,j}) +  I_{\cal X}(\bX) + \tr\left(\lam^{(k)} (\bX - \bD^{(k)}\bA^{(k)}  )  \right) + \frac{\rho}{2} \|\bX - \bD^{(k)}\bA^{(k)} \|_F^2,
\end{eqnarray} 
which after completing the square and ignoring constant terms is equivalent to 
\begin{eqnarray}\label{eqn:opt_S1}
 \min_{\bX\in\mathbb{R}^{n_1\times n_2}} && \sum_{i,j} s_{i,j}\ell (Y_{i,j}, X_{i,j}) +  I_{\cal X}(\bX) + \frac{\rho}{2} \left\| \bX - \bD^{(k)}\bA^{(k)} + \frac{ \lam^{(k)} }{\rho}    \right\|_F^2. 
\end{eqnarray}
Due to the assumed separability of the indicator function, the above problem is separable in each entry $X_{i,j}$ and the entries can be updated in parallel by solving the following scalar convex optimization problem for each entry.
When $\ell(y, x)$ is a convex function of $x$, the solution is given by
\begin{eqnarray}
\nonumber    X_{i,j}^{(k+1)} &=&  \textrm{Proj}_{{\cal X}_{i,j}} \left[  \textrm{prox}_{s_{i,j},\ell} \left( (\bD^{(k)}\bA^{(k)})_{i,j} - \frac{ (\lam^{(k)})_{i,j} }{\rho}  ; \rho, Y_{i,j}  \right) \right],\\
    &=& \begin{cases}
    \textrm{Proj}_{{\cal X}_{i,j}} \left[  \textrm{prox}_{\ell} \left( (\bD^{(k)}\bA^{(k)})_{i,j} - \frac{ (\lam^{(k)})_{i,j} }{\rho}  ; \rho, Y_{i,j}  \right) \right], & \text{if} ~s_{i,j} = 1\\
    \textrm{Proj}_{{\cal X}_{i,j}} \left[ (\bD^{(k)}\bA^{(k)})_{i,j} - \frac{ (\lam^{(k)})_{i,j} }{\rho} \right],              & \text{otherwise}
\end{cases}
\end{eqnarray}
where $\textrm{prox}_{\ell} (  z ; \rho, y ) =  \arg \min_{x \in\mathbb{R}} \ \ell (y, x) + \frac{\rho}{2} \left( x - z \right)^2$ is the \emph{proximal operator} of the loss function $\ell (y, \cdot)$ and ${\rm Proj}_{{\cal X}_{i,j}}(x)$ is the projection\footnote{Here, this set is just an interval and the projection operator returns $x$ if $x \in {\cal X}_{i,j}$ or the nearest endpoint of the interval ${\cal X}_{i,j}$ otherwise.} of the scalar $x$ onto the set ${\cal X}_{i,j}$. For several of the loss functions we consider, the proximal operator can be computed in closed form; for the one-bit settings we can use Newton's second order method (or gradient descent) to numerically evaluate it as described later. A table of proximal operators for the various losses we consider here is provided in Table~\ref{tab:prox_formulae}.

\begin{table}[t!]
	\begin{center}
 		\begin{tabular}{| c | c | c |}
 				\hline
 				  & $\ell (y, x)$  & $\textrm{prox}_{\ell} (  z ; \rho, y )$  \\ \hline \hline
				 \vspace{-1.3em} & &  \\
 				\textbf{Gaussian} & $\frac{(y-x)^2}{2\sigma^2}$ & $ \frac{y + \sigma^2 \rho z }{1+\sigma^2\mu} $ \\
				\vspace{-1.3em} & &  \\ \hline
 				\vspace{-1.3em} & &  \\
				\textbf{Poisson} & $x - y \log(x) $ & $\frac{\rho z - 1 +  \sqrt{(\rho z - 1)^2 + 4 \rho y}}{2\rho} $  \\ 
				\vspace{-1.3em} & &  \\ \hline
 				\textbf{Laplace} & $\lambda |y-x| $ & $  y - \textrm{Soft}(z-y, \lambda/ \rho ) $ \\ \hline
 					\textbf{One-bit} & $-y \log ( F(x)) - (1-y)\log ( 1-F(x))   $ &  (\emph{Newton's method}) \\ \hline
 			\end{tabular}
 			 	\caption{Expressions for $\textrm{prox}_{\ell} (  z ; \rho, y ) =  \arg \min_{x \in\mathbb{R}} \ \ell (y, x) + \frac{\rho}{2} \left( x - z \right)^2$ for different $\ell(y, x)$, corresponding to negative log-likelihoods for the models we examine.  Here, for $\lambda>0$, ${\rm soft}(x,\lambda) = {\rm sgn}(x)\max\{|x|-\lambda, 0\}$}
				\label{tab:prox_formulae}
		\end{center}
\end{table}    

Completing the square and ignoring the constant terms the subproblem $\mathbf{S2}$ is equivalent to 
\begin{eqnarray}\label{eqn:opt_S2}
\bA^{(k+1)} = {\rm arg}~\min_{\bA\in\mathbb{R}^{r \times n_{2}}} &&   I_{\cal A}(\bA) + \lambda \|\bA\|_0 + \frac{\rho}{2} \left\| \bX^{(k+1)} - \bD^{(k)}\bA + \frac{ \lam^{(k)} }{\rho}   \right\|_F^2.
\end{eqnarray}
In order to solve this problem we adopt the constrained iterative hard thresholding approach from \cite{Lu:12}, as outlined in Algorithm \ref{alg:IHT}.  Finally, after completing the square and ignoring the constant terms, we see that the subproblem $\mathbf{S3}$ is equivalent to 
\begin{eqnarray}\label{eqn:opt_S3}
\bD^{(k+1)} = {\rm arg} \min_{\bD\in\mathbb{R}^{n_{1}\times r}} &&  I_{\cal D}(\bD) + \frac{\rho}{2} \left\| \bX^{(k+1)} - \bD\bA^{(k+1)} + \frac{ \lam^{(k)} }{\rho}    \right\|_F^2,
\end{eqnarray}
which we solve here by projected Newton gradient descent algorithm, described in Algorithm \ref{alg:D_Newton}. Our overall algorithmic approach is summarized in Algorithm \ref{alg:ADMM}.

\begin{algorithm}[h]
	\caption{A\_IHT($\bX$, $\bD$, $\bZ$,~$\epsilon$)  -- For solving $ \min_{\bA\in\mathbb{R}^{n_1\times p}}   I_{\cal A}(\bA) + \lambda \|\bA\|_0 + \frac{\rho}{2} \left\| \bZ - \bD\bA \right\|_F^2$}
	\label{alg:IHT}
	\begin{algorithmic} 
		\STATE \hspace{-1.4em} \textbf{Inputs:} $\bX,\bD,\bZ, \epsilon, \rho $
		\STATE \hspace{-1.4em} \textbf{Initialize:} $\bA^{(0)} = \mathbf{0}$
		\REPEAT
		\STATE  $\bY^{(k+1)} =  \bA^{(k)} - \bD^T(\bD\bA^{(k)} - \bZ )/\| \bD \|_2^2 $
		\STATE \textbf{Update:}  $Y^{(k+1)}_{i,j} = 0 $ if  $ \big | Y^{(k+1)}_{i,j} \big |\le  \sqrt{\frac{2\lambda}{\rho \| \bD \|_2^2}}$.
		\STATE $\quad \text{if}~Y^{(k+1)}_{i,j} \in {\cal A}_{i,j}:$
		\STATE $\quad \quad \quad A^{(k+1)}_{i,j} = Y^{(k+1)}_{i,j};$
		\STATE $\quad \text{else:}$
		\STATE $\quad \quad \quad A^{(k+1)}_{i,j} = {\rm arg}~\underset{x \in {\cal A}_{i,j}}{{\rm min}} ~ \left\{x^{2} - 2x\left[ A^{(k)}_{i,j} - \frac{((\bD^T\bD\bA^{(k)})_{i,j} - (\bD^T\bZ)_{i,j} )}{\| \bD \|_2^2}\right]\right\}$
		\UNTIL  $   \frac{\| \bA^{(k+1)} - \bA^{(k)} \|_F }{ \| \bA^{(k)}\|_F} \le \epsilon $
		\STATE \hspace{-1.4em} \textbf{Output:} $\bA = \bA^{(k+1)}$
	\end{algorithmic}
\end{algorithm}

\begin{algorithm}[h]
	\caption{D\_Newton($\bX$, $\bA$, $\bZ$,~$\epsilon$)  -- For solving $ \min_{\bD\in\mathbb{R}^{n_1\times p}}   I_{\cal D}(\bD) + \frac{\rho}{2} \left\| \bZ - \bD\bA \right\|_F^2$}
	\label{alg:D_Newton}
	\begin{algorithmic} 
		\STATE \hspace{-1.4em} \textbf{Inputs:} $\bX, \bA, \bZ, \epsilon, \rho $
		\STATE \hspace{-1.4em} \textbf{Initialize:} $\bD^{(0)} = \mathbf{0}$
		\REPEAT
		\STATE $\bD^{(k+1)} = \textrm{Proj}_{\cal D}\left[  \bD^{(k)} - \rho\left(\bD^{(k)}\bA - \bZ \right)\bA^T \left( \rho\bA\bA^T + \delta \bI\right )^{-1}   \right] $
		\UNTIL  $   \frac{\| \bD^{(k+1)} - \bD^{(k)} \|_F }{ \| \bD^{(k)}\|_F} \le \epsilon $
		\STATE \hspace{-1.4em} \textbf{Output:} $\bD = \bD^{(k+1)}$\\ 
	\end{algorithmic}
\end{algorithm}

\begin{algorithm}[h]
	\caption{ ADMM algorithm for solving problem \eqref{eq:mainOpt} }
	\label{alg:ADMM}
	\begin{algorithmic} 
		\STATE \hspace{-1.4em} \textbf{Inputs:} $\epsilon_1, \epsilon_2, ~\Delta_{1}, ~\Delta_{2},~\Delta_{1}^{\rm stop}, ~\Delta_{2}^{\rm stop},~\eta,~ \rho^{(0)}>0$
		\STATE \hspace{-1.4em} \textbf{Initialize:} $\bD^{(0)} \in \cal D$ , $\bA^{(0)} \in \cal A$, $\lam^{(0)}$. 
		\REPEAT
		\STATE   $\bX_{i,j}^{(k+1)}  =  \textrm{Proj}_{\cX} \left[  \textrm{prox}_{s_{i,j}\ell} \left( (\bD^{(k)}\bA^{(k)})_{i,j} - \frac{ (\lam^{(k)})_{i,j} }{\rho^{(k)}}  ; \rho^{(k)}, Y_{i,j}  \right) \right] $
		\STATE $\bA^{(k+1)}$ := A\_IHT $\left( \bX_{(k+1)},\bD^{(k)}, \bX^{(k+1)} + \lam^{(k)}/\rho^{(k)} , \epsilon_1\right)$
		\STATE  $\bD^{(k+1)} :=$  D\_Newton $\left( \bX_{(k+1)},\bA^{(k+1)}, \bX^{(k+1)} + \lam^{(k)}/\rho^{(k)} , \epsilon_2  \right)$ 
		\STATE  $\lam^{(k+1)} = \lam^{(k)} + \rho^{(k)} (\bX^{(k+1)} - \bD^{(k+1)}\bA^{(k+1)})$
		\STATE Set $\Delta_{1} = \| \bX^{(k+1)} - \bD^{(k+1)}\bA^{(k+1)}  \|_F $ and $\Delta_{2} =  \rho^{(k)} \cdot \|\bD^{(k)}\bA^{(k)}  - \bD^{(k+1)}\bA^{(k+1)}  \|_F$
		\STATE $\rho^{(k+1)}= 
\begin{cases}
    \eta \cdot \rho^{(k)},& \text{if } \Delta_{1} \geq 10 \cdot \Delta_{2}\\
    \rho^{(k)}/\eta,              & \text{if } \Delta_{2} \geq 10 \cdot \Delta_{1}\\
    \rho^{(k)}, 			& \text{otherwise}
\end{cases}$

		\UNTIL  $ \Delta_{1}  \le \Delta_{1}^{\rm stop} $ and $  \Delta_{2} \le \Delta_{2}^{\rm stop} $
		\STATE \hspace{-1.4em} \textbf{Output:} $\bD = \bD^{(k+1)} \text{ and }\bA = \bA^{(k+1)}$
	\end{algorithmic}
\end{algorithm}

\subsection{Experiments}

We perform experimental validation of our theoretical results on synthetic data for two different scenarios, corresponding to when the columns of the matrix $\bA^{*}$ are $k$-sparse, and when each belongs to a weak-$l_{p}$ ball. For each scenario we construct the true data matrices $\bX^{*} = \bD^{*}\bA^{*}$ by individually constructing the matrices $\bD^{*} $ and $\bA^{*}$ (as described below), where the entries of the true matrices $\bX^{*}$, $\bD^{*}$, and $\bA^{*}$ are bounded in $[{\rm X}_{\rm{min}}^{*},~ {\rm X}_{\rm{max}}^{*}]$, $[{\rm D}_{\rm{min}}^{*}, ~{\rm D}_{\rm{max}}^{*}]$ and $[{\rm A}_{\rm{min}}^{*},~ {\rm A}_{\rm{max}}^{*}]$ respectively.  


We generate the $\bD^{*}$ matrix by first generating a Gaussian random matrix of size $n_{1} \times r$ whose entries are distributed as ${\cal N}(0, 1)$, then multiplying each element by $({\rm D}_{\rm max}^{*} - {\rm D}_{\rm min}^{*})$ to avoid pathological scaling issues. Finally, we project the resulting scaled matrix onto the set $\cal D$, which here is done by truncating all the entries bigger than ${\rm D}_{\rm max}^{*}$ to ${\rm D}_{\rm max}^{*}$ and truncating all the entries smaller than ${\rm D}_{\rm min}^{*}$ to ${\rm D}_{\rm min}^{*}$.  We construct sparse $\bA^{*}$ by generating a Gaussian random matrix of size $r \times n_{2}$, multiplying it by $({\rm A}_{\rm max}^{*} - {\rm A}_{\rm min}^{*})/3$, and projecting it onto the set $\cal A$. Then we randomly select $r - k$ locations from each column of the resulting matrix and set the corresponding entries to $0$. For the approximately sparse $\bA^{*}$, we generate each column to be a randomly permuted version of $\{{\rm A}_{\rm max}^{*}\cdot i^{-1/p}\}_{i = 1}^{r}$ with random signs (except for the Poisson likelihood case, where each column of $\bA^{*}$ has nonnegative elements).

We define the set ${\cal X}$ such that each entry of $\bX$ is bounded in the range $[{\rm X}_{\rm{min}},~ {\rm X}_{\rm{max}}]$, $\cal D$ and $\cal A$ are the set of all matrices $\bD \in \mathbb{R}^{n_{1} \times r}$ and $\bA \in \mathbb{R}^{r \times n_{2}}$ whose entries are bounded in the range $[{\rm D}_{\rm{min}}, ~{\rm D}_{\rm{max}}]$ and $[{\rm A}_{\rm{min}},~ {\rm A}_{\rm{max}}]$ respectively. It is important to note that in general the actual bounds on the magnitude of the entries of true matrices (for e.g., ${\rm D}_{\rm{min}}^{*}$, ${\rm A}_{\rm{max}}^{*}$ etc.) are unknown, and therefore during optimization we might have to use their approximations (which here are denoted as ${\rm D}_{\rm{min}}$, ${\rm A}_{\rm{max}}$ etc.) to define the feasible sets $\cal {X}$, $\cal D$ and $\cal A$.  Our specific choices of parameters for the four different likelihoods considered in this paper are summarized in Table~\ref{parameters}.

\begin{table}[h]
	\begin{center}
 		\begin{tabular}{| c | c | c | c | c|}
 				\hline
 				 \textbf{Parameters $\backslash$ Likelihood}  & \textbf{Gaussian} & \textbf{Laplace} & \textbf{Poisson} & \textbf{One-bit} \\ \hline \hline
 				$n_{1} \times n_{2}$ & $100 \times 1000$ & $100 \times 1000$ & $100 \times 1000$ & $1000 \times 1000$ \\ \hline
				$r,~k,~p$ & $20,~8,~1/3$  & $20,~8,~1/3$ & $20,~8,~1/3$ & $5,~2,~1/3$ \\ \hline
 				$[{\rm D}_{\rm min}^{*}, ~{\rm D}_{\rm max}^{*}]$ & $[-1,~1]$  & $[-1,~1]$ & $[0.1,~1]$ & $[-1,~1]$ \\ \hline
 				$[{\rm A}_{\rm min}^{*}, ~{\rm A}_{\rm max}^{*}]$ & $[-20,~20]$  & $[-20,~20]$ & $[0,~40]$ & $[-20,~20]$ \\ \hline
				$[{\rm D}_{\rm min}, ~{\rm D}_{\rm max}]$ & $[-2,~2]$  & $[-2,~2]$ & $[-2,~2]$ & $[-2,~2]$ \\ \hline
				$[{\rm A}_{\rm min}, ~{\rm A}_{\rm max}]$ & $[-40,~40]$  & $[-40,~40]$ & $[-80,~80]$ & $[-40,~40]$ \\ \hline
				$[{\rm X}_{\rm min}, ~{\rm X}_{\rm max}]$ & $[-2\cdot{\rm X}_{\rm min}^{*},~2\cdot{\rm X}_{\rm max}^{*}]$  & $[-2\cdot{\rm X}_{\rm min}^{*},~2\cdot{\rm X}_{\rm max}^{*}]$ & $[0,~2\cdot{\rm X}_{\rm max}^{*}]$ & $[-2\cdot{\rm X}_{\rm min}^{*},~2\cdot{\rm X}_{\rm max}^{*}]$ \\ \hline
 			\end{tabular}
 			 	\caption{Experimental parameters for different likelihood models we examine. Here ${\rm X}_{\rm min}^{*} = \min_{i,j} X^*_{i,j}$ and ${\rm X}_{\rm max}^{*} = \|\bX^{*}\|_{\rm max}$. }
				\label{parameters}
		\end{center}
\end{table}   

%
%
%
%
%

Now, our experimental approach is as follows.  For sparse and nearly-sparse (with columns belonging to a weak $\ell_p$ ball with $p=1/3$) coefficient matrices $\bA^*$ we generate a corresponding matrix $\bX^*$ as above.  Then, for each of a number of regularization parameters $\lambda>0$ and and sampling rates $\gamma \in (0, 1]$ we perform 20 trials of the following experiment: we generate $\cS$ according to the independent Bernoulli($\gamma$) model, obtain noisy observations of $\bX^{*}$ according to the \eqref{eqn:obsmodel}, use Algorithm \ref{alg:ADMM} to obtain\footnote{For Algorithm \ref{alg:ADMM} we set $\epsilon_{1} = \epsilon_{2} = 10^{-7}$, $\Delta_{1}^{\rm stop} = \Delta_{2}^{\rm stop} = 10$, $\eta = 1.05$ and $\rho^{(0)} = 0.001$.} an estimate $\widehat{\bX} = \widehat{\bD}\widehat{\bA}$, and compute its approximation error $\frac{\|\widehat{\bX} - \bX^{*}\|_{F}^{2}}{n_{1}n_{2}}$.  We then compute the empirical average of the errors over the 20 trials for each setting.  Fig.~\ref{figure:results1} shows the results of this experiment for the Gaussian, Laplace and Poisson likelihood models. The plots depict the empirical average (over 20 trials) per-element error as a function of sampling rate on a log-log scale; the curves shown are corresponding to the best (lowest) errors achieved over all of the regularization parameters $\lambda$ we examined. The first row corresponds to exactly sparse $\bA^{*}$ matrices and the plots in the second row corresponds to settings where $\bA^*$ is approximately sparse.  The three columns correspond to three different regimes for the Gaussian and Laplace settings (we chose the parameters $\sigma$ and $\tau$ for the Gaussian and Laplace settings, respectively, to yield identical variances; the first column corresponds to $\sigma = 0.5$ and $\tau = \sqrt{8}$, the second column corresponds to $\sigma = 1$ and $\tau = \sqrt{2}$, and the third column corresponds to $\sigma=2$ and $\tau = 1/\sqrt{2}$).  

\begin{figure*}[t]
\centering
\begin{tabular}{ccc}
\epsfig{file=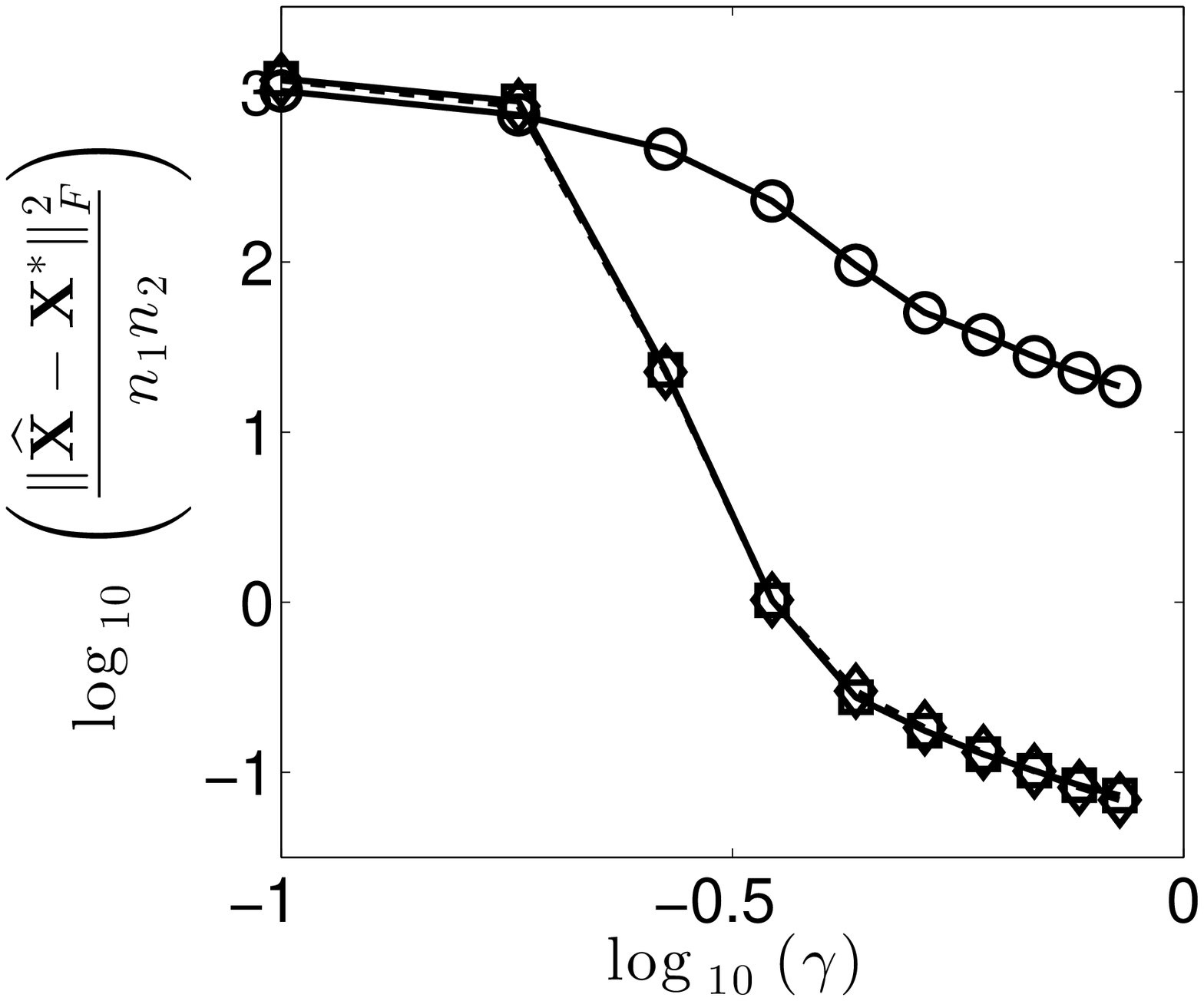,width=0.3\linewidth,clip=} &
\epsfig{file=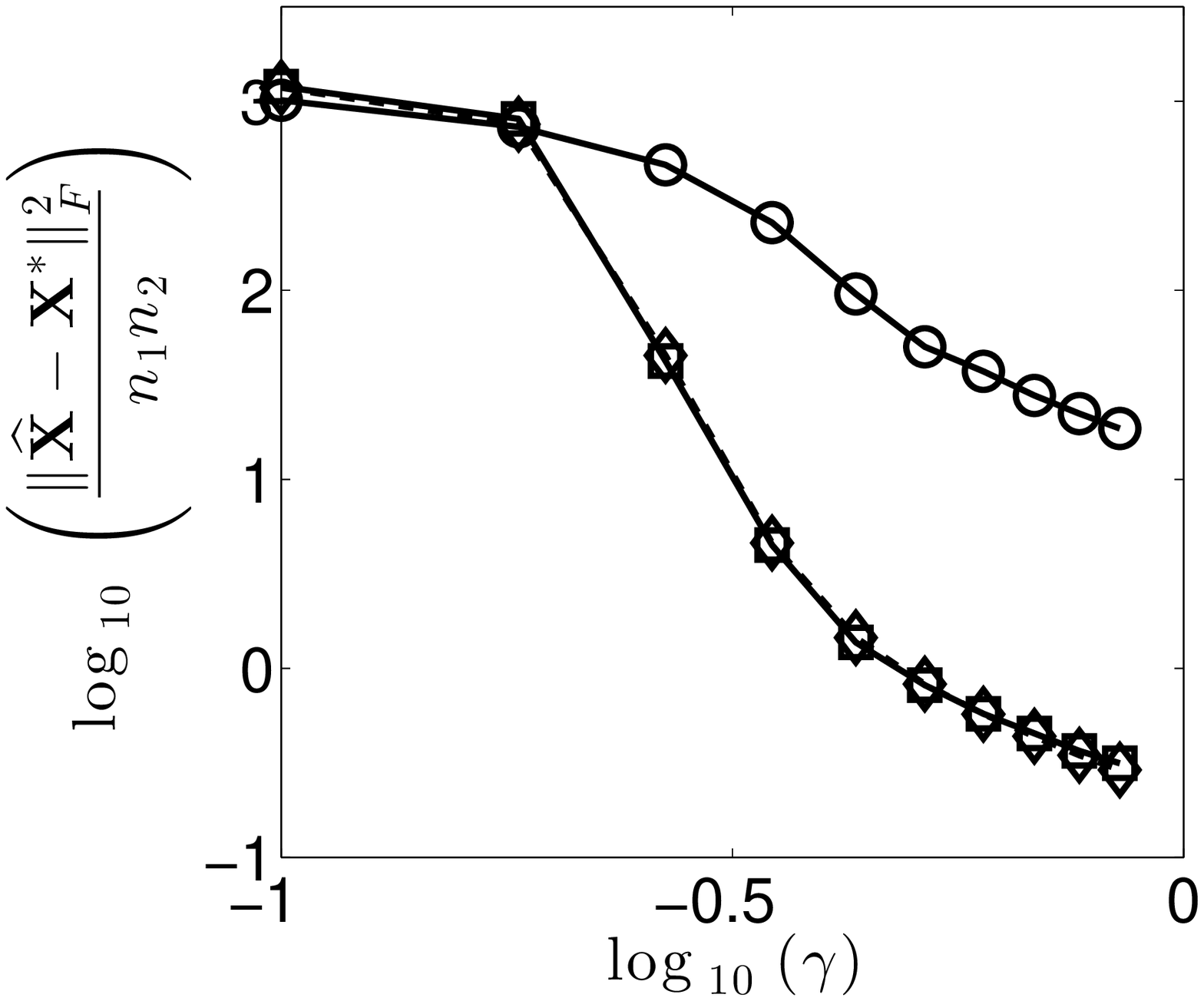,width=0.3\linewidth,clip=} &
\epsfig{file=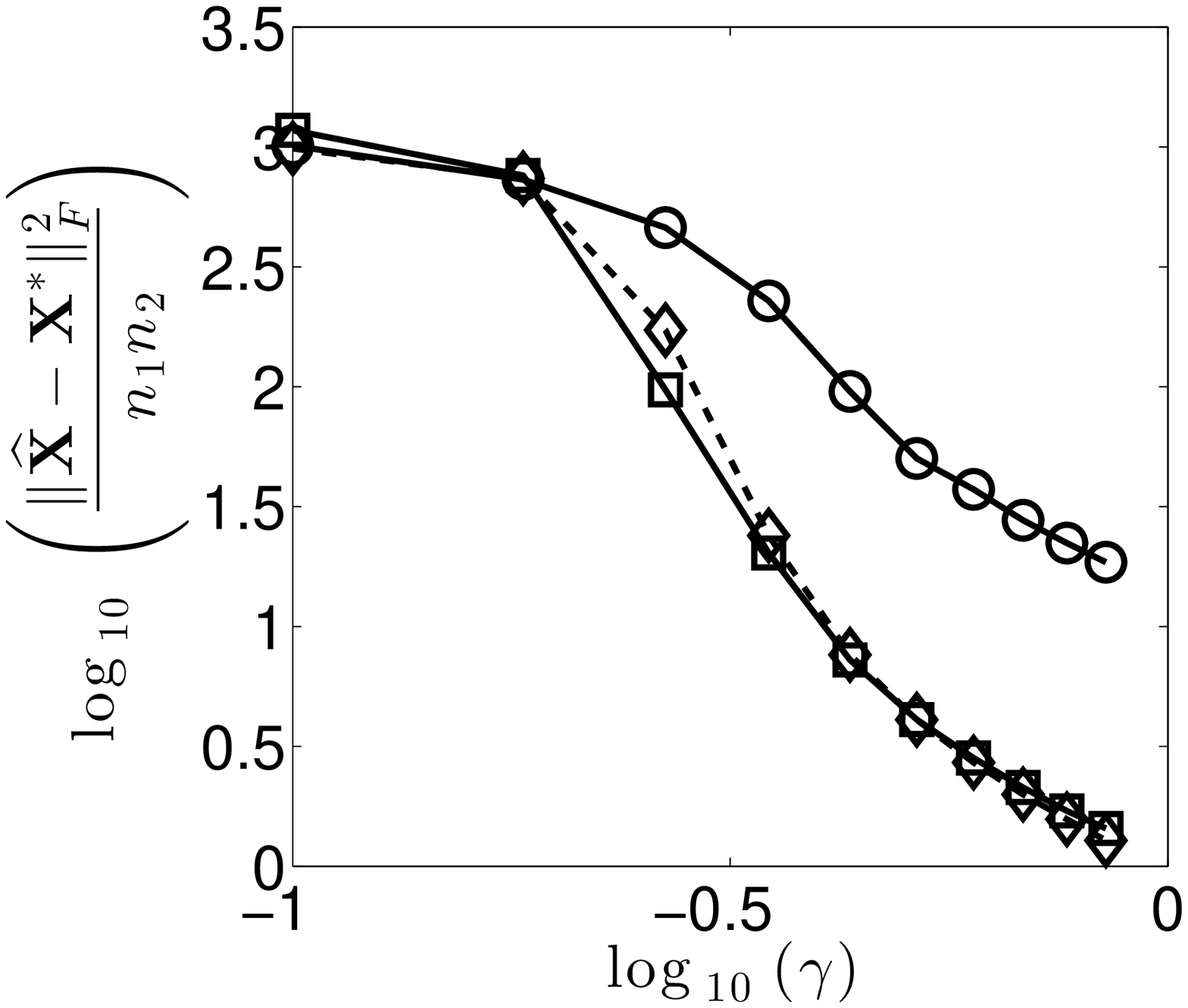,width=0.3\linewidth,clip=}\\
\epsfig{file=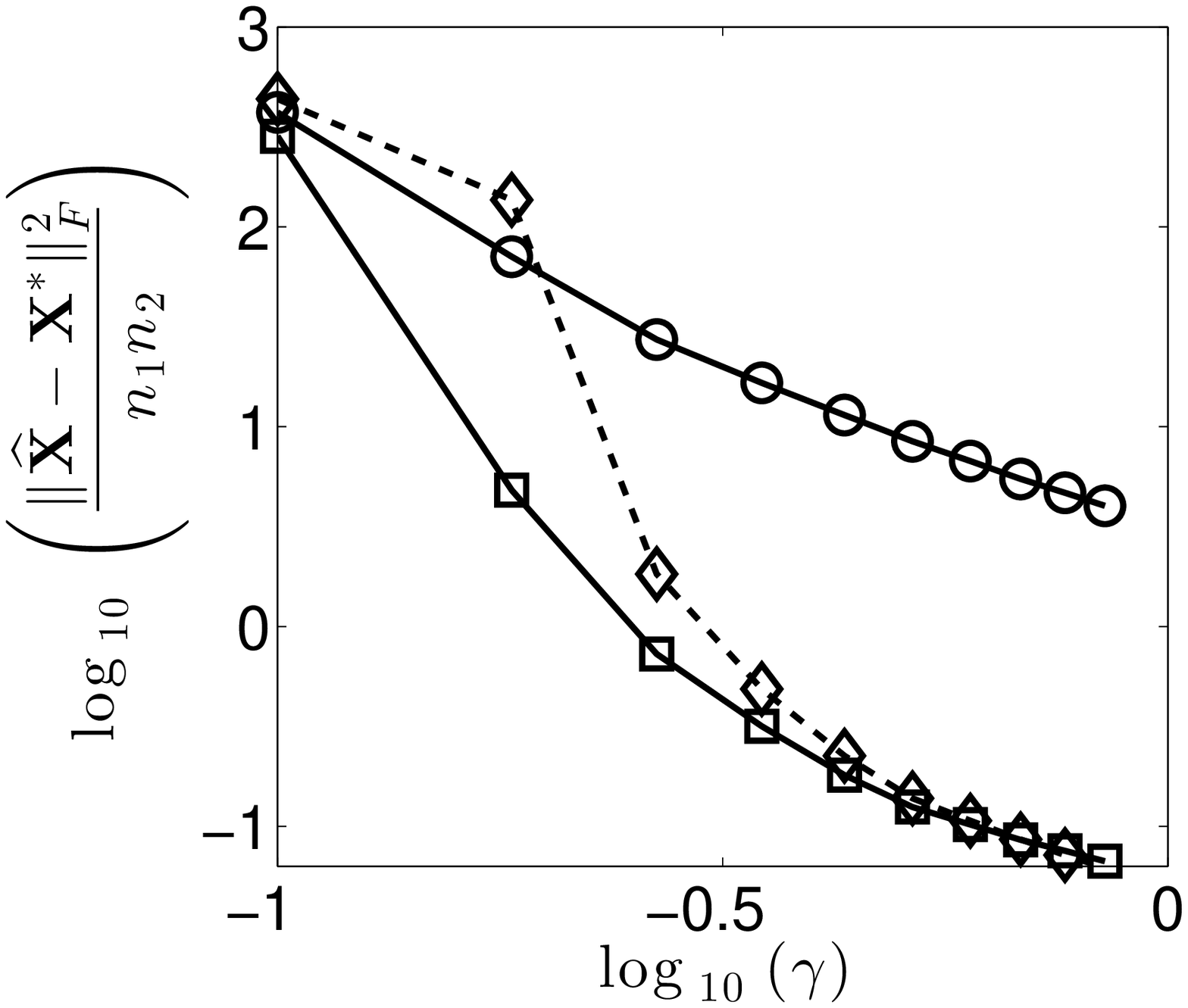,width=0.3\linewidth,clip=}&
\epsfig{file=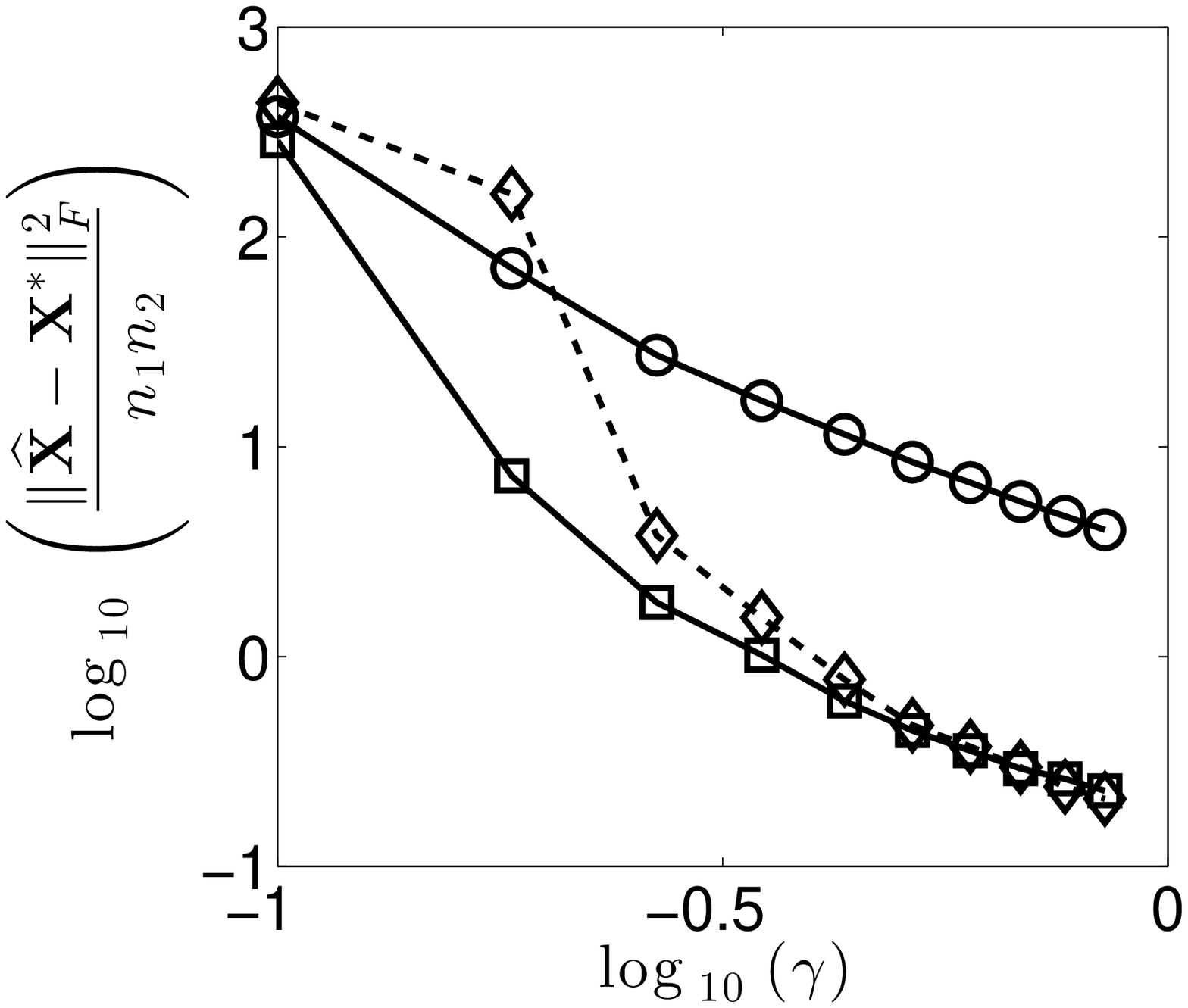,width=0.3\linewidth,clip=}&
\epsfig{file=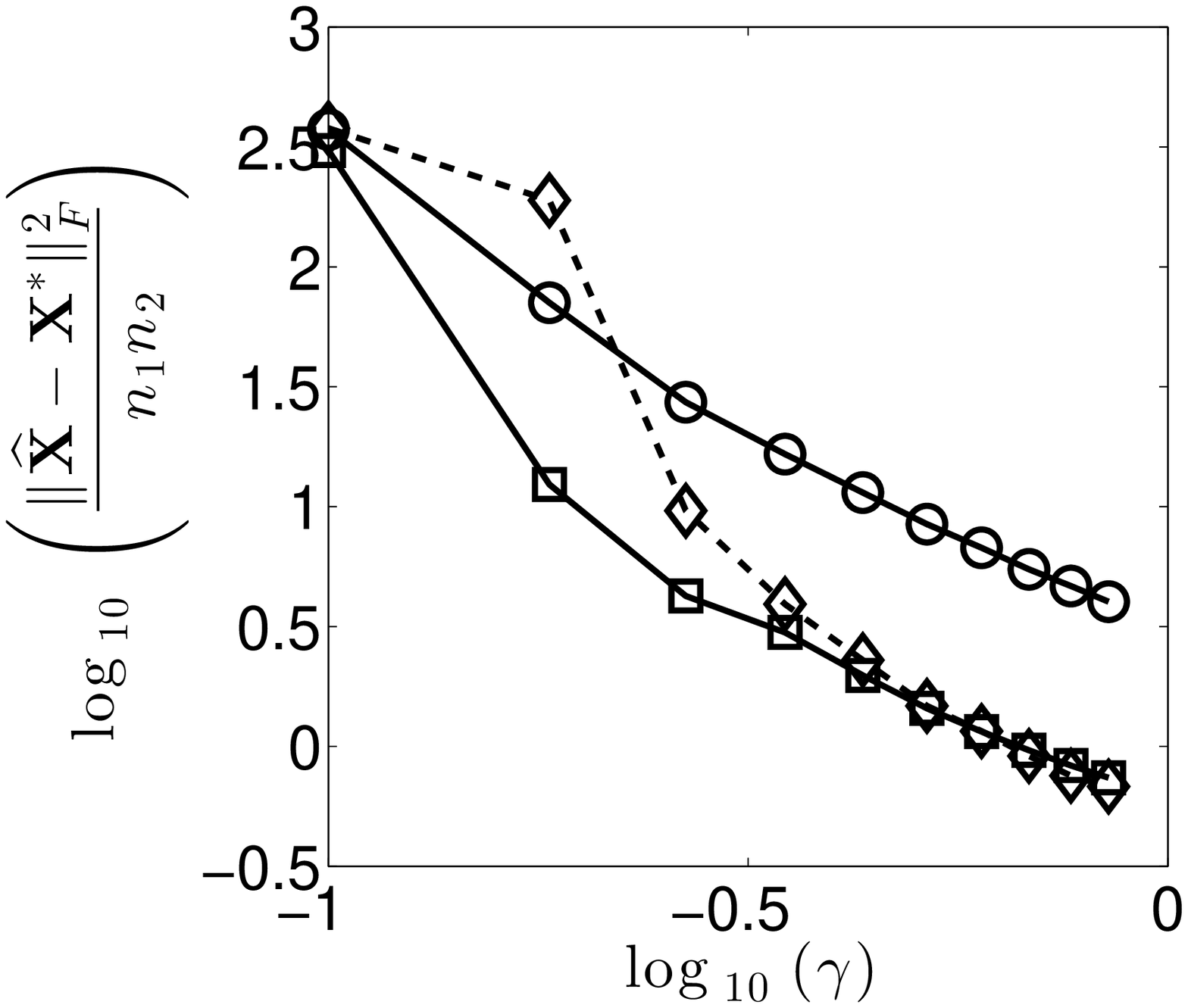,width=0.3\linewidth,clip=} 
\end{tabular}
\caption{Results of synthetic experiments for matrix completion with Gaussian, Laplace and Poisson likelihoods: ----$\Box$---- is Gaussian, $--\Diamond--$ is Laplace, and ----$\circ$---- is Poisson. Top row corresponds to sparse-factor model with $k = 8$ while the bottom row corresponds to weak-$l_{p}$ model with $p = 1/3$. Column $1$ corresponds to $\sigma^{2} = (0.5)^{2}$ (for Laplace $\tau = \sqrt{8}$), column $2$ corresponds to $\sigma^{2} = (1)^{2}$ (for Laplace $\tau = \sqrt{2}$) and column $3$ corresponds to $\sigma^{2} = (2)^{2}$ (for Laplace $\tau = 1/\sqrt{2}$). Here $n_{1} = 100$, $n_{2} = 1000$ and $r = 20$.}
\label{figure:results1}
\end{figure*}

A few interesting points are worth noting here. First, for the case where $\bA^*$ is exactly sparse, our theoretical results predict the error decay be inversely proportional to the nominal sampling rate $\gamma$; viewed on a log-log scale, this would correspond to the error decay having slope -$1$.  Our experimental results provide some evidence to validate our analysis, at least in the settings where the sampling rate $\gamma > 0.4$ -- there, the slopes of the error decays for each of the likelihood models is indeed approximately -$1$.  For the settings where the columns of $\bA^*$ belong to a weak-$\ell_p$ (with $p=1/3)$ our theory predicts that the slope of the error decay (on a log-log scale) be at least $-5/6$ for the Gaussian-noise and Poisson-distributed cases, and at least ($-2/3$) for the Laplace-noise case.  For our experiments here, it appears that the error decay in these approximately-sparse settings is actually a bit faster than predicted by the theory, as the error appears to decay with a slope of approximately -$1$.  That said, it is worth noting that our predicted rate in these cases was obtained essentially by a (squared) bias-variance tradeoff, so quantify a kind of worst-case behavior that may not always be observed in practice.

We also evaluated the performance in this setting for a one-bit observation model, using an analogous experimental setting as above.  Here, we used the logistic cumulative distribution function as the link function, i.e., $F(x) = \frac{1}{1 + e^{-x/s}}$ where $s = \frac{\sqrt{3}\cdot \sigma}{\pi}$ and $\sigma$ is a parameter that could be viewed as additive noise standard deviation\footnote{For this link function the proximal operator is 
$\textrm{prox}_{\ell} (  z ; \rho, y ) =  \arg \min_{x \in\mathbb{R}} \ -y\log(F(x)) -(1-y)\log(1 - F(x))+ \frac{\rho}{2} \left( x - z \right)^2$,
which in general is not solvable in closed form. Here, we resort to Newton's gradient descent algorithm -- rewriting the problem as $\textrm{prox}_{\ell} (  z ; \rho, y ) =  \arg \min_{x \in\mathbb{R}} \ G(x)$,
where $G(x) = -y\log(F(x)) -(1-y)\log(1 - F(x))+ \frac{\rho}{2} \left( x - z \right)^2$, it is easy to show that the gradient is $\nabla G(x) = -\frac{y}{s} + \frac{F(x)}{s} + \rho(x - z)$ and the Hessian is $\nabla^{2}G(x) = \frac{F(x)(1-F(x))}{s^{2}} + \rho$. We can then iteratively solve for $\textrm{prox}_{\ell} (  z ; \rho, y )$ by Newton steps (starting from a random $x^{(0)}$) of the form $x^{(k+1)} = x^{(k)} -  \frac{\nabla G(x)}{\nabla^{2} G(x)} $ until convergence (here, until $\|x^{(k+1)} - x^{(k)}\| \leq 10^{-7}$).
}, for the specific choice $\sigma = 0.1$.  Fig.~\ref{figure:results2} shows the error results for this case, with the first plot corresponds to sparse $\bA^{*}$ and the second to when each column of $\bA^{*}$ lies in a weak-$\ell_{p}$ ball with $p = 1/3$. As in the previous experiments, it appears here that the slope of the error decay is approximately -$1$ in each case.  Note that we adapted the experimental setting here to be more amenable to this more difficult estimation regime (specifically, we consider slightly larger matrices but having smaller rank and fewer nonzeros per column of the factor $\bA^*$, as outlined in Table~\ref{parameters}, so that the number of observations per parameter to be estimated is larger than in the previous three experimental settings). 

\begin{figure*}[t ]
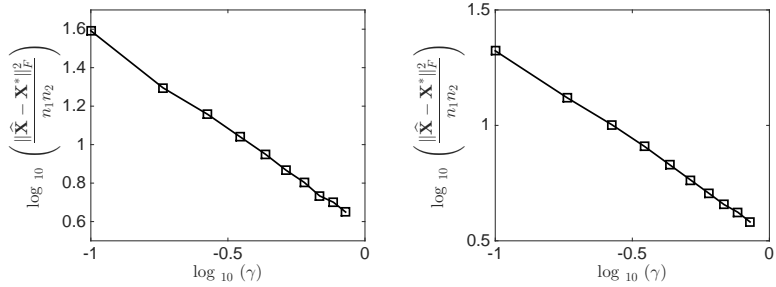

\centering
\begin{tabular}{cc}
\epsfig{file=NonDecayLogit.eps,width=0.3\linewidth,clip=} &
\epsfig{file=DecayLogit.eps,width=0.3\linewidth,clip=} 
\end{tabular}
\caption{Results of synthetic experiments for one-bit matrix completion under sparse factor models, using the logistic link function.  The left panel corresponds to the case where the $\bA^*$ matrix is exactly sparse; the right when columns of $\bA^*$ lie in a weak-$l_{p}$ ball with $p = 1/3$.}
\label{figure:results2}
\end{figure*}


\section{Discussion and Conclusions}\label{sec:disc}

We conclude with a brief discussion of our results and potentially interesting future directions.

\subsection{Extensions to Other Data Models}

Each of our theoretical results above follow essentially from the specialization of a more general result (appearing below as Lemma~\ref{lem:main}) to the case of sparse factor models.  It is interesting to note that this lemma may also be specialized (in a straightforward manner) to any of a number of other interesting factor models (e.g., non-negative matrix factorizations, factorizations where each factor may be sparse, etc.) under the same general observation models we consider here.  Further, while we provide Lemma~\ref{lem:main} specifically for the case of matrix completion, the essential analysis extends (simply) to higher-order structures (i.e., tensors) as well.   

%

\subsection{Convexification?}

As discussed in several points in the preceding sections, the optimization associated with the estimators we consider here is non-convex on account of several factors, including the presence of the $\ell_0$ term in the objective, our optimization over a discretized set, and more fundamentally, the fact that we perform inference in a general bilinear model, where both factors are unknown.  Resolving ourselves, then, to seek only local optima of the corresponding optimizations allows us to bring to bear alternating direction method of multipliers techniques, in which the $\ell_0$-based optimization \emph{subproblems} may be solved efficiently.  Interestingly, within this framework we may also directly incorporate the constraints that the matrix factor elements each come from a discretized set (indeed, this would correspond to choosing set indicator functions that take the value $\infty$ outside of the discretized sets over which we seek to optimize).  We did not pursue this latter condition in our simulations, assuming instead that the discretization of each of the elements be ``sufficiently fine'' so that we may solve the optimization numerically at machine precision (and replace the discretized sets for the candidate matrix factors by their convex hulls).  

The fact that we can (locally) handle the $\ell_0$ constraints within the ADMM framework notwithstanding, it is interesting to consider whether there is any benefit to relaxing this constraint to a convex surrogate (e.g., replacing the $\ell_0$ penalty with an $\ell_1$ penalty).  The resulting procedure would still be jointly non-convex in the matrix factors, but could be addressed within a similar algorithmic framework to the one we propose above.   Analytically, methods that prescribe optimization over a convex set comprised of the Cartesian product of a set ${\cal D}$ of matrices whose elements satisfy a max-norm constraint and a set ${\cal A}$ of matrices whose columns satisfy an $\ell_1$-constraint may be amenable to analysis using entropy-based methods that can be employed to analyze estimation error performance by bounding suprema of empirical processes indexed by elements of the feasible set of candidate estimates -- see, e.g., \cite{VanDeGeer:00, Koltch:11book, Boucheron:13}.  It would be interesting to see whether analyses along these yield substantially different results than our analysis here; analyses along these lines are a subject of our ongoing work and will be reported in a subsequent effort. 

In the meantime, it is interesting to examine (albeit, empirically) whether our algorithmic approach yields significantly different performance if we replace the $\ell_0$ regularization term by an $\ell_1$ term.  To provide some insight into this, we consider a problem of completing a $50\times 500$ matrix $\bX^*=\bD^* \bA^*$, where $\bD^*$ is $50\times 10$ and $\bA^*$ is $10\times 500$ and sparse, having $4$ nonzero elements per column.  We consider Gaussian noise-corrupted observations obtained at a subset of locations of $\bX^*$ (generated according to the independent Bernoulli model), and three different reconstruction approaches: the first is the algorithmic approach described in the previous section, the second is a slight variation of our proposed approach where we replace the $\ell_0$ penalty by an $\ell_1$ penalty (and replace the corresponding inference step with an accelerated first-order method as in \cite{Beck:09}), and the the third method is a more standard low-rank recovery obtained via \emph{nuclear-norm} regularization, as $\widehat{\bX} = \arg\min_{\bX} \|\bY_{\cS}-\bX_{\cS}\|_F^2 + \lambda \|\bX\|_{*}$.  For each method, we examined a range of possible values for the regularization parameter, and selected the reconstruction corresponding (clairvoyantly) to the best choice for each method.  The results, provided in Figure~\ref{figure:images}, show that the best-performing $\ell_0$ and $\ell_1$ regularized sparse factor completion methods perform comparably, while both achieve (slightly) lower error than the best nuclear norm regularized completion estimate.  Of course, as noted above, our algorithmic approach identifies (at best) a local minimum of the overall non-convex problem we aim to solve, but even at that, it is encouraging to see that the ADMM-based optimization(s) identify good-quality estimates. 
\begin{figure*}[t]
\centering
\begin{tabular}{c}
\epsfig{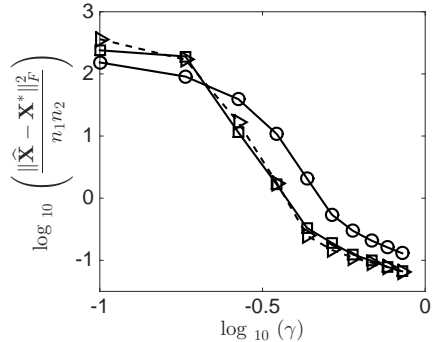} 
\end{tabular}
\caption{Comparison between sparse-factor and nuclear-norm-regularized matrix completion methods.  The curves are: our proposed procedure with $\ell_0$ regularizer ($\Box$), the $\ell_1$ regularized variant of our approach ($\rhd$), and nuclear norm regularized low-rank matrix completion ($\circ$).  The sparse factor completion methods perform similarly, and both achieve a lower error than the best nuclear-norm regularized estimate for sampling rates $\gamma \geq 10^{-0.5} \approx 30\%$.}
\label{figure:images}
\end{figure*}

It is also interesting to consider an alternative, more essential, convexification of our problem of interest here, using the machinery of \emph{atomic norms} as introduced in \cite{Chand:12}.  Specifically, one may view matrices adhering to the sparse factor models we investigate here as sums of rank-one matrices formed as outer products between a (non-sparse) $n_1\times 1$ vector and a (sparse) $n_2\times 1$ vector.  Following \cite{Chand:12}, one can consider the convex hull of the set of all such rank-one atoms having unit (Frobenius) norm as the unit-ball for a norm that serves as a regularizer for matrices representable by weighted sums of only a few atoms.  A very recent work \cite{Richard:14} has begun to identify properties of atomic norms so-formed, and extensions to the cases where both of the vectors may be sparse, and have established some estimation guarantees for recovering simple matrices (comprised of a single rank-one outer product of sparse vectors) from a collection of Gaussian measurements.  Interestingly, the authors of \cite{Richard:14} note that resulting inference procedures using their so-called $(k,q)$-norm (formed from atoms that are rank-one outer products between $k$-sparse and $q$-sparse vectors), while convex, may still be computationally intractable (even NP-hard)!  At any rate, it would be quite interesting to extend this approach to the entry-wise sampling models and various likelihood models we consider here, and we defer investigations along these lines to a future work. 

\subsection{Lower Bounds}

Our error bounds here provide some insight into the performance of sparsity-penalized maximum likelihood estimation approaches to sparse factor matrix completion tasks.  To the best of our knowledge, lower bounds on the achievable mean-square estimation error for these tasks have not been established, but would be a valuable complement to place our results here into a broader context.  Efforts along these lines are ongoing, and will be reported in a future work.


\appendix

\subsection{Proof of Theorem~\ref{thm:main}}\label{a:thmproof}

Our proof of Theorem~\ref{thm:main} is based on an application of the following general lemma, which we prove in Appendix~\ref{a:proof}. 

\begin{lemmai}\label{lem:main}
Let $\bX^*$ be an $n_1\times n_2$ matrix whose elements we aim to estimate, and let $\cX$ be a countable collection of candidate reconstructions $\bX$ of $\bX^*$, each with corresponding penalty $\pen(\bX)\geq 1$, so that the collection of penalties satisfies the summability condition $\sum_{\bX\in\cX} 2^{-\pen(\bX)} \leq 1$. 

Fix an integer $m$ with $4\leq m \leq n_1n_2$, let $\p = m(n_1n_2)^{-1}$, generate a sampling set $\cS$ according to the independent Bernoulli($\p$) model so that each $(i,j)\in[n_1]\times [n_2]$ is included in $\cS$ independently with probability $\p$, and obtain corresponding observations $\bY_{\cS}\sim p_{\bX^*_{S}}= \prod_{(i,j)\in\cS} p_{X^*_{i,j}}$, which are assumed to be conditionally independent given $\cS$.  Then, if $\cD$ is any constant satisfying
\begin{equation}
\cD \geq \max_{\bX\in\cX} \max_{(i,j)\in[n_1]\times [n_2]}  \D(p_{X^*_{i,j}}\|p_{X_{i,j}}),
\end{equation} 
we have that for any 
\begin{equation}\label{eqn:xicond}
\xi \geq \left(1 + \frac{2\cD}{3}\right) \cdot 2\log 2, 
\end{equation}
the complexity penalized maximum likelihood estimator
\begin{equation}
\widehat{\bX}^{\xi} = \widehat{\bX}^{\xi}(\cS, \bY_{\cS}) = \arg\min_{\bX\in \cX}\left\{-\log p_{\bX_{\cS}}(\bY_{\cS}) + \xi \cdot \pen(\bX)  \right\},
\end{equation}
satisfies the (normalized, per-element) error bound
\begin{equation}
\frac{\E_{\cS,\bY_{\cS}}\left[-2\log \A(p_{\widehat{\bX}^{\xi}}, p_{\bX^*})\right]}{n_1 n_2} \leq 3 \cdot \min_{\bX\in\cX} \left\{ \frac{\D(p_{\bX^*}\|p_{\bX})}{n_1 n_2} + \left(\xi + \frac{4\cD\log 2}{3}\right)\frac{\pen(\bX)}{m}\right\} + \frac{8 \cD \log m}{m},
\end{equation}
where, as denoted, the expectation is with respect to the joint distribution of $\cS$ and $\bY_{\cS}$.
\end{lemmai}



In order to use this result here, we need to define penalties $\pen(\bX)\geq 1$ on candidate reconstructions $\bX$ of $\bX^*$, so that for every subset $\cX$ of the set $\cX'$ specified in the conditions of Theorem~\ref{thm:main} the summability condition $\sum_{\bX\in\cX} 2^{-\pen(\bX)} \leq 1$ holds. To this end, we will use the fact that for any $\cX \subseteq \cX'$ we always have $\sum_{\bX \in \cX} 2^{-\pen(\bX)} \leq \sum_{\bX \in \cX'} 2^{-\pen(\bX)}$; thus, it suffices for us to show that for the specific set $\cX'$ described in Section~\ref{sec:prob},
\begin{equation}\label{eqn:Kraft}
\sum_{\bX\in\cX'} 2^{-\pen(\bX)} \leq 1.
\end{equation}
Note that the condition \eqref{eqn:Kraft} is the well-known Kraft-McMillan Inequality for coding elements of $\cX'$ with an alphabet of size $2$, which is satisfied automatically if we choose the penalties to be \emph{code lengths} for some uniquely decodable binary code for the elements $\bX\in\cX'$; see \cite{Cover:12}.  This interpretation will provide us with a \emph{constructive} approach to designing penalties, as we will see below.  

Now, consider any discretized matrix factors $\bD\in {\cal D}$ and $\bA\in{\cal A}$, as described in Section~\ref{sec:prob}.  Let us fix an ordering of the indices of elements of $\bD$ and encode the amplitude of each element using $\log_2 L_{\rm lev}$ bits, and for $L_{\rm loc} \triangleq 2^{\lceil \log_2 r n_2 \rceil}$ we encode each \emph{nonzero} element of $\bA$ using $ \log_2 L_{\rm loc}$ bits to denote its location and $\log_2 L_{\rm lev}$ bits for its amplitude.  With this strategy, a total of $n_1 r \log_2 L_{\rm lev}$ bits are used to encode $\bD$ and matrices $\bA$ having $\|\bA\|_{0}$ nonzero entries are encoded using $\|\bA\|_{0} (\log_2 L_{\rm loc} + \log_2 L_{\rm lev})$ bits.  Now, we let $\cX''$ be the set of all such $\bX=\bD\bA$, and let the code for each $\bX$ be the concatenation of the (fixed-length) code for $\bD$ followed by the (variable-length) code for $\bA$. It follows that we may assign penalties $\pen(\bX)$ to all $\bX\in \cX''$ whose lengths satisfy
\begin{equation}
\pen(\bX) = n_1 r \log_2 L_{\rm lev} + \|\bA\|_{0} (\log_2 L_{\rm loc} + \log_2 L_{\rm lev}).
\end{equation}
It is easy to see that such codes are (by construction) uniquely decodable, so we have that $\sum_{\bX\in\cX''} 2^{-\pen(\bX)} \leq 1$.  Now, the set $\cX'$ specified in the theorem is a subset of $\cX''$ (or perhaps $\cX''$ itself, if all elements satisfy the max norm bound condition $\|\bX\|_{\rm max} \leq {\rm X}_{\rm max}$), so \eqref{eqn:Kraft} holds for $\cX'$ as specified in the theorem.

Now let $\cX$ be any subset of $\cX'$.  By the above argument the summability condition holds for $\cX$, so we may apply the results of Lemma~\ref{lem:main}.   For randomly subsampled and noisy observations $\bY_{\cS}$ our estimates take the form
\begin{eqnarray}
\nonumber \widehat{\bX}^{\xi} &=& \arg\min_{\bX=\bD\bA \in \cX}\left\{-\log p_{\bX_{\cS}}(\bY_{\cS}) + \xi \cdot \pen(\bX)  \right\}\\
&=&  \arg\min_{\bX=\bD\bA\in \cX}\left\{-\log p_{\bX_{\cS}}(\bY_{\cS}) + \xi \cdot  (\log_2 L_{\rm loc} + \log_2 L_{\rm lev}) \cdot \|\bA\|_{0}  \right\}.
\end{eqnarray}
where the last line follows by disregarding additive constants in the optimization arising from terms that do not depend on $\bX$ (or more specifically, on $\bD$ or $\bA$) in the penalty. Further, when $\xi$ satisfies \eqref{eqn:xicond},  we have
\begin{eqnarray}
\nonumber \lefteqn{\frac{\E_{\cS,\bY_{\cS}}\left[-2\log \A(p_{\widehat{\bX}^{\xi}}, p_{\bX^*})\right]}{n_1 n_2} \leq \frac{8 \cD \log m}{m} +}&&\\
&&3 \cdot \min_{\bX\in\cX} \left\{ \frac{\D(p_{\bX^*}\|p_{\bX})}{n_1 n_2} + \left(\xi + \frac{4\cD\log 2}{3}\right)(\log_2 L_{\rm loc} + \log_2 L_{\rm lev})\left(\frac{n_1 r + \|\bA\|_0}{m}\right)\right\},
\end{eqnarray}
Finally, letting 
\begin{eqnarray}
\lambda &=& \xi \cdot  (\log_2 L_{\rm loc} + \log_2 L_{\rm lev})
\end{eqnarray}
and using the fact that
\begin{equation}
\log_2 L_{\rm loc} + \log_2 L_{\rm lev}  \leq (\beta + 2)  \cdot \log(n_1\vee n_2) \cdot 2\log 2
\end{equation}
which follows by our selection of $L_{\rm lev}$ and $L_{\rm loc}$ and the fact that $r< n_2$, it follows (after some straightforward simplification) that for 
\begin{equation}
\lambda \geq 2 (\beta + 2) \left(1+\frac{2\cD}{3}\right) \log(n_1 \vee n_2)
\end{equation}
the estimate 
\begin{equation}
\nonumber \widehat{\bX}^{\lambda} = \arg\min_{\bX\in \cX}\left\{-\log p_{\bX_{\cS}}(\bY_{\cS}) + \lambda \cdot \|\bA\|_{0}  \right\}
\end{equation}
satisfies
\begin{eqnarray}
\nonumber \lefteqn{\frac{\E_{\cS,\bY_{\cS}}\left[-2\log \A(p_{\widehat{\bX}^{\lambda}}, p_{\bX^*})\right]}{n_1 n_2} \leq \frac{8 \cD \log m}{m} +}\hspace{5em}&&\\
&&3 \cdot \min_{\bX\in\cX} \left\{ \frac{\D(p_{\bX^*}\|p_{\bX})}{n_1 n_2} + \left(\lambda + \frac{4\cD (\beta + 2) \log(n_1\vee n_2)}{3}\right)\left(\frac{n_1 r + \|\bA\|_0}{m}\right)\right\},
\end{eqnarray} 
as claimed.

\subsection{Proof of Corollary~\ref{cor:Gauss}}\label{a:gaussproof}

We first establish a general error bound, which we then specialize to the case stated in the corollary.  Note that for $\bX^*$ as specified and any $\bX\in\cX$, using the model \eqref{eqn:likGauss} we have
\begin{equation}
\D(p_{X_{i,j}^*}\|p_{X_{i,j}}) =  \frac{(X_{i,j}^*-X_{i,j})^2}{2\sigma^2}
\end{equation}
for any fixed $(i,j)\in S$. It follows that $\D(p_{\bX^*}\|p_{\bX})=\|\bX^*-\bX\|_F^2/2\sigma^2$, and using the fact that  the amplitudes of entries of $\bX^*$ and all $\bX\in\cX$ are no larger than $\Xmax$, it is clear that we may choose $\cD = 2 \Xmax^2/\sigma^2$.  Further, for any $\bX\in\cX$ and any fixed $(i,j)\in\cS$ it is easy to show that in this case
\begin{equation}
-2\log \A(p_{X_{i,j}}, p_{X_{i,j}^*}) =  \frac{(X_{i,j}^*-X_{i,j})^2}{4\sigma^2},
\end{equation}
so that $-2\log \A(p_{\bX}, p_{\bX^*})=\|\bX^*-\bX\|_F^2/4\sigma^2$.  It follows that 
\begin{equation}
\E_{\cS, \bY_{\cS}} \left[ -2 \log \A(p_{\widehat{\bX}},p_{\bX^*}) \right] = \frac{\E_{\cS, \bY_{\cS}} \left[ \|\bX^*-\widehat{\bX}\|_F^2 \right]}{4\sigma^2}.
\end{equation}
Incorporating this into Theorem~\ref{thm:main}, we obtain that for any
\begin{equation}
\lambda \geq \left(1 + \frac{4\Xmax^2}{3\sigma^2}\right) \cdot 2 (\beta + 2) \cdot \log(n_1 \vee n_2), 
\end{equation}
the sparsity penalized ML estimate satisfies the per-element mean-square error bound
\begin{eqnarray}\label{eqn:Gaussbnd}
\nonumber  \lefteqn{\frac{\E_{\cS,\bY_{\cS}}\left[\|\bX^*-\widehat{\bX}\|_F^2\right]}{n_1 n_2} \leq  \frac{64 \Xmax^2 \log m}{m}}\hspace{0em} &&\\
&+& 6 \cdot \min_{\bX\in\cX} \left\{ \frac{\|\bX^*-\bX\|_F^2}{n_1 n_2}  + \left(2\sigma^2 \lambda + \frac{16\Xmax^2 (\log 2)^2 (\beta+1)\log (n_1\vee n_2)}{3}\right)\left(\frac{n_1 p + \|\bA\|_0}{m}\right)\right\}.
\end{eqnarray}

We now establish the error bound for the case where the coefficient matrix $\bA^*$ is exactly sparse and $\lambda$ is fixed to the value specified in \eqref{eqn:lamchoose}.  Consider a candidate reconstruction of the form $\bX^*_{Q} = \bD^*_{Q}\bA^*_{Q}$, where the elements of $\bD^*_{Q}$ are the closest discretized surrogates of the entries of $\bD^{*}$, and the entries of and $\bA^*_{Q}$  are the closest discretized surrogates of the \emph{nonzero} entries of $\bA^{*}$ (and zero otherwise). Denote $\bD^*_{Q} = \bD^{*} + \triangle_{\bD^*}$ and $\bA^*_{Q} = \bA^{*} + \triangle_{\bA^*}$. Then it is easy to see that
\begin{equation}
\bD^*_{Q}\bA^*_{Q} - \bD^{*}\bA^{*} = \bD^{*}\triangle_{\bA^*} + \triangle_{\bD^*}\bA^{*} + \triangle_{\bD^*}\triangle_{\bA^*}.
\end{equation}
Given the range limits on allowable $\bD$ and $\bA$ and that each range is quantized to $L_{\rm lev}$ levels, we have that $\|\triangle_{\bD^*}\|_{\rm max} \leq 1/(L_{\rm lev} -1)$ and $\|\triangle_{\bA^*}\|_{\rm max} \leq \Amax/(L_{\rm lev} -1)$. Now, we can obtain a bound on the magnitudes of the elements of $\bD^*_{Q}\bA^*_{Q} - \bD^{*}\bA^{*}$ that hold uniformly over all $i,j$, as follows
\begin{eqnarray}\label{eqn:quantmax}
\nonumber \|\bD^*_{Q}\bA^*_{Q} - \bD^{*}\bA^{*}\|_{\rm max} &=& \max_{i,j} |(\bD^{*}\triangle_{\bA^*} + \triangle_{\bD^*}\bA^{*} + \triangle_{\bD^*}\triangle_{\bA^*})_{i,j}|\\
\nonumber &\leq&  \max_{i,j} |(\bD^{*}\triangle_{\bA^*})_{i,j}| + |(\triangle_{\bD^*}\bA^{*})_{i,j}| + |(\triangle_{\bD^*}\triangle_{\bA})_{i,j}|\\
\nonumber &\leq& \frac{r\Amax}{L_{\rm lev}-1} + \frac{r\Amax}{L_{\rm lev}-1} + \frac{2r\Amax}{(L_{\rm lev}-1)^2}\\
&\leq& \frac{8r\Amax}{L_{\rm lev}},
\end{eqnarray}
where the first inequality follows from the triangle inequality, the second from the bounds on $\|\triangle_{\bD^*}\|_{\rm max}$ and $\|\triangle_{\bA^*}\|_{\rm max}$ and the entry-wise bounds on elements of allowable $\bD$ and $\bA$, and the last because $L_{\rm lev} \geq 2$.  Now, it is straight-forward to show that our choice of $\beta$ in \eqref{eqn:beta} implies $L_{\rm lev} \geq 16r \Amax/\Xmax$, so each entry of $\bD^*_{Q}\bA^*_{Q} - \bD^{*}\bA^{*}$ is bounded in magnitude by $\Xmax/2$.  It follows that each element of the candidate $\bX^*_{Q}$ constructed above is bounded in magnitude by $\Xmax$, so $\bX^*_Q$ is indeed a valid element of the set $\cX$. 

Further, the approximation error analysis above also implies directly that 
\begin{eqnarray}
\nonumber \frac{\|\bX^{*} - \bX^*_{Q}\|_{F}^{2}}{n_1n_2} &=& \frac{1}{n_1 n_2} \sum_{i \in [n_{1}], j \in [n_{2}]} (\bD^*_{Q}\bA^*_{Q} - \bD^{*}\bA^{*})^{2}_{i,j} \\
\nonumber &\leq& \frac{64 p^{2} \Amax^2}{L_{\rm lev}^{2}}\\
&\leq& \frac{\Xmax^2}{m} , 
\end{eqnarray}
where the last line follows from the fact that our specific choice of $\beta$ in \eqref{eqn:beta} also implies $L_{\rm lev} \geq 8 r \sqrt{m} \Amax/\Xmax$. Now, evaluating the oracle term at the candidate $\bX^*_Q=\bD^*_Q\bA^*_Q$, and using the fact that $\|\bA^*_Q\|_0 = \|\bA^*\|_0$, we have
\begin{eqnarray}
\frac{\E_{\cS,\bY_{\cS}}\left[\|\bX^*-\widehat{\bX}\|_F^2\right]}{n_1 n_2} \leq \frac{70 \Xmax^2 \log m}{m} + 8 (3\sigma^2 + 8\Xmax^2) (\beta+2)\log (n_1\vee n_2)\left(\frac{n_1 r + \|\bA^*\|_0}{m}\right).
\end{eqnarray} 

Finally, we establish the error bound for the case where columns of $\bA^*$ are in a weak $\ell_p$ ball of radius $\Amax$, for $p\leq 1$.  To that end, let us denote the columns of $\bA^*$ by $\ba^*_j$ for $j\in[n_2]$, and for any $k\in[r]$, we let $\ba^{*,(k)}_j$ denote the best $k$-term approximation of $\ba^{*}_j$, formed by retaining the largest (in magnitude) elements and setting the rest to zero.  For shorthand, we denote by $\bA^{*,(k)}$ the matrix with columns $\ba^{*,(k)}_j$ for $j\in[n_2]$.  Now, the approximation error incurred may be bounded as
\begin{eqnarray}
\nonumber \|\bX^{*} - \bX^{*,(k)}\|_{F}^{2} &=& \sum_{i,j}(\bD^{*}(\bA^{*} - \bA^{*,(k)}))_{i,j}^{2}\\
&\leq& \sum_{i,j} \  \|\ba_{j}^{*} - \ba_{j}^{*,(k)}\|_2^2,
\end{eqnarray}
where the inequality follows from the fact that each $(\bD^{*}(\bA^{*} - \bA^{*,(k)}))_{i,j}$ may be expressed as an inner product between the $i$-th row of $\bD^*$ (whose elements are no larger than $1$ in magnitude) and the $j$-th column of $\bA^{*} - \bA^{*,(k)}$. To simplify further, we use the fact that $p \leq 1$ (and $q\geq 2p$), and the approximation behavior of vectors in weak $\ell_p$ balls (discussed in the preliminaries) to obtain that $\|\ba_{j}^{*} - \ba_{j}^{*,(k)}\|_2^2 \leq \Amax^2 k^{-2(1/p-1/2)}$. Letting $\alpha=1/p-1/2$, we have that the approximation error associated with approximating $\bA^*$ by its best $k$-term approximation satisfies $\|\bX^{*} - \bX^{*,(k)}\|_{F}^{2} \leq n_1 n_2 \Amax^2 k^{-2\alpha}$.

Now, we consider a candidate reconstruction of the form $\bX^{*,(k)}_Q=\bD^*_Q\bA_Q^{*,(k)}$ where $\bD^*_Q$ is as above and where the nonzero elements of $\bA^{*,(k)}_Q$ are taken to be the closest quantized surrogates of the corresponding nonzero elements of $\bA^{*,(k)}$.  Using the fact that 
\begin{eqnarray}
\nonumber \frac{\|\bX^* - \bX_{Q}^{*,(k)}\|_F^2}{n_1 n_2} &\leq& \frac{4\left(\|\bX^* - \bX^{*,(k)}\|_F^2 + \|\bX^{*,(k)} - \bX_Q^{*,(k)}\|_F^2\right)}{n_1 n_2}\\
&\leq& 4 \Amax^2 k^{-2\alpha} + \frac{4 \Xmax^2}{m},
\end{eqnarray}
where the first term on the bottom results from the approximation error analysis above and the second from our analysis of the first result of the corollary, 
we evaluate the oracle bound at the candidate $\bX^{(k)}_Q$ to obtain
\begin{eqnarray}\label{eqn:sGaussbnd}
\lefteqn{\frac{\E_{\cS,\bY_{\cS}}\left[\|\bX^*-\widehat{\bX}\|_F^2\right]}{n_1 n_2}} &&\\
\nonumber  &\leq & \frac{88 \Xmax^2 \log m}{m} + \min_{k\geq 1}\left\{24 \Amax^2 k^{-2\alpha} + 8(3\sigma^2 + 8\Xmax^2) (\beta+2)\log (n_1\vee n_2)\left(\frac{n_1 r + kn_2}{m}\right)\right\}.
\end{eqnarray}
Finally, we choose $k = (m/n_2)^{1/(1+2\alpha)}$ to balance the decay rates on the $k^{-2\alpha}$ and $kn_2 /m$ terms, and thus obtain
\begin{eqnarray}\label{eqn:cGaussbnd}
\nonumber \lefteqn{\frac{\E_{\cS,\bY_{\cS}}\left[\|\bX^*-\widehat{\bX}\|_F^2\right]}{n_1 n_2} \leq \frac{88 \Xmax^2 \log m}{m} + 8 (3\sigma^2 + 8\Xmax^2)  (\beta+2)\log (n_1\vee n_2) \frac{n_1 r}{m}}\hspace{10em}&&\\
&+& \left[24\Amax^2 + 8(3\sigma^2 + 8\Xmax^2) (\beta+2)\log (n_1\vee n_2)\right] \left(\frac{n_2}{m}\right)^{\frac{2\alpha}{2\alpha+1}}.
\end{eqnarray}

The stated bounds in each case follow from some straight-forward bounding, as well as the fact mentioned in Section~\ref{sec:main}, that under our assumptions, $(\beta+2)\log(n_1\vee n_2) = {\cal O}(\log(n_1\vee n_2))$. 

\subsection{Proof of Corollary~\ref{cor:Lap}}\label{a:lapproof}

We follow a similar approach as in the proof of Corollary~\ref{cor:Gauss}, and first establish the general error bound.  For $\bX^*$ as specified and any fixed $\bX\in\cX$. We have by (relatively) straight-forward calculation that for any fixed $(i,j)\in S$,
\begin{eqnarray}
\nonumber \D(p_{X_{i,j}^*}\|p_{X_{i,j}}) &=&  \tau \ |X_{i,j}^*-X_{i,j}| - (1-e^{-\tau \ |X_{i,j}^*-X_{i,j}|})\\
&\leq&  \tau \ |X_{i,j}^*-X_{i,j}|
\end{eqnarray}
where the inequality follows from the fact that $(1-e^{-\tau \ |X_{i,j}^*-X_{i,j}|}) \geq 0$, and
\begin{eqnarray}
\nonumber -2\log \A(p_{X_{i,j}}, p_{X_{i,j}^*}) &=&  \tau \ |X_{i,j}^*-X_{i,j}| - 2\log\left(1+ \tau \ \frac{|X_{i,j}^*-X_{i,j}|}{2}\right)\\
&\geq& \frac{\tau^2}{4(\tau \Xmax + 1)^2} (X_{i,j}^*-X_{i,j})^2,
\end{eqnarray}
where the inequality follows from the convexity of the negative log Hellinger affinity along with an application of Taylor's theorem\footnote{Formally, letting $x\triangleq X^*_{i,j} - X_{i,j}$ and $f(x) = \tau |x| - 2\log(1 + \tau |x|/2)$ we have 
\begin{equation*}
f'(x) = \frac{\tau^2}{2}\left(\frac{x}{1+\tau|x|/2}\right) \ \mbox{ and } \ f''(x) = \frac{\tau^2}{2(1+\tau|x|/2)^2}.
\end{equation*}
Thus, $f(x)$ is twice differentiable (everywhere). The result follows from the fact that $f(0)=f'(0) = 0$ and 
\begin{equation*}
f''(x) \geq \frac{\tau^2}{2(1+\tau\Xmax)^2}
\end{equation*}
for all $x$ of the specified form, given the assumptions on $\bX^*$ and $\bX$.}.  It follows from this that $\D(p_{\bX^*}\|p_{\bX}) \leq \tau \ \|\bX^*-\bX\|_1$, and
\begin{equation}
\E_{\cS, \bY_{\cS}} \left[ -2 \log \A(p_{\widehat{\bX}},p_{\bX^*}) \right] \geq \frac{\tau^2}{4(\tau \Xmax + 1)^2} \ \E_{\cS, \bY_{\cS}} \left[ \|\bX^*-\widehat{\bX}\|_F^2 \right].
\end{equation}
Further, we may choose $\cD = 2\tau \Xmax$.  Incorporating this into Theorem~\ref{thm:main}, we have that for any
\begin{equation}
\lambda \geq 2 (\beta + 2) \left(1 + \frac{4\tau \Xmax}{3}\right) \log(n_1\vee n_2), 
\end{equation}
the sparsity-penalized ML estimate satisfies
\begin{eqnarray}\label{eqn:lapgenbnd}
\lefteqn{\frac{\E_{\cS,\bY_{\cS}}\left[\|\bX^*-\widehat{\bX}\|_F^2\right]}{n_1 n_2} \leq  \frac{1}{\tau}\cdot \frac{64 (\tau \Xmax + 1)^2 \Xmax \log m}{m} + }\hspace{0em} &&\\
\nonumber  && \hspace{-0.5em}\frac{12(\tau\Xmax + 1)^2}{\tau} \cdot \min_{\bX\in\cX} \left\{ \frac{ \|\bX^*-\bX\|_1}{n_1 n_2}  + \left(\frac{\lambda}{\tau} + \frac{8\Xmax (\beta+2)\log (n_1\vee n_2)}{3}\right)\left(\frac{n_1 p + \|\bA\|_0}{m}\right)\right\}.
\end{eqnarray}

We now establish the error bound for the case where the coefficient matrix $\bA^*$ is sparse and $\lambda$ is fixed to the value \eqref{eqn:lamchoose}.  We again consider a candidate reconstruction of the form $\bX^*_{Q} = \bD^*_{Q}\bA^*_{Q}$, where the elements of $\bD^*_{Q}$ are the closest discretized surrogates of the entries of $\bD^{*}$, and the entries of and $\bA^*_{Q}$  are the closest discretized surrogates of the nonzero entries of $\bA^{*}$ (and zero otherwise).  Now, since $\beta$ is the same as in the proof of Corollary~\ref{cor:Gauss}, we can directly apply the bound of \eqref{eqn:quantmax} (and use the fact that $L_{\rm lev} \geq 16 r \Amax/\Xmax\}$) to conclude that 
\begin{equation}
\frac{\|\bX^*-\bX^*_Q\|_1}{n_1n_2} \leq \frac{\Xmax}{2n_1n_2} \leq \frac{\Xmax}{m}.
\end{equation}  
Now, evaluating the oracle term at the candidate $\bX^*_Q=\bD^*_Q\bA^*_Q$, and using the fact that $\|\bA^*_Q\|_0 = \|\bA^*\|_0$, we have
\begin{eqnarray}
\nonumber  \lefteqn{\frac{\E_{\cS,\bY_{\cS}}\left[\|\bX^*-\widehat{\bX}\|_F^2\right]}{n_1 n_2} \leq \frac{76 (\tau \Xmax + 1)^2}{\tau^2} \cdot \frac{\tau \Xmax \log m}{m}}\hspace{4em}&&\\
&& +  \frac{12(\tau\Xmax + 1)^2}{\tau^2} \left(2 + \frac{16\tau\Xmax}{3}\right) (\beta+2)\log (n_1\vee n_2)\left(\frac{n_1 p + \|\bA^*\|_0}{m}\right).
\end{eqnarray} 

Finally, we establish the error bound for the case where the columns of $\bA^*$ are vectors in a weak $\ell_p$ ball for $p\leq 1/2$.  By a similar analysis as above, we conclude that $\|\bX^{*} - \bX^{*,(k)}\|_1 \leq n_1 n_2 \Amax k^{-\alpha'}$, where $\alpha=1/p-1$.  Now, we consider a candidate reconstruction of the form $\bX^{*,(k)}_Q=\bD^*_Q\bA_Q^{*,(k)}$ where $\bD^*_Q$ is as above and where the nonzero elements of $\bA^{*,(k)}_Q$ are taken to be the closest quantized surrogates of the corresponding nonzero elements of $\bA^{*,(k)}$.  Using the fact that 
\begin{eqnarray}
\nonumber \frac{\|\bX^* - \bX_{Q}^{*,(k)}\|_1}{n_1 n_2} &\leq& \frac{\|\bX^* - \bX^{*,(k)}\|_1 + \|\bX^{*,(k)} - \bX_Q^{*,(k)}\|_1}{n_1 n_2}\\
&\leq& \Amax k^{-\alpha'} + \frac{\Xmax}{m},
\end{eqnarray}
where the first term on the bottom results from the approximation error analysis above and the second from our analysis of the first result of the corollary, 
we evaluate the oracle bound at the candidate $\bX^{(k)}_Q$ to obtain
\begin{eqnarray}
\lefteqn{\frac{\E_{\cS,\bY_{\cS}}\left[\|\bX^*-\widehat{\bX}\|_F^2\right]}{n_1 n_2} \leq \frac{76 (\tau \Xmax + 1)^2}{\tau^2} \cdot \frac{\tau \Xmax \log m}{m}}\hspace{4em}&&\\
\nonumber && +  \frac{12(\tau\Xmax + 1)^2}{\tau^2} \min_{k\geq 1}\left\{\tau \Amax k^{-\alpha'} + \left(2 + \frac{16\tau\Xmax}{3}\right) (\beta+2)\log (n_1\vee n_2)\left(\frac{n_1 p + n_2 k}{m}\right)\right\}.
\end{eqnarray} 
Finally, we choose $k = (m/n_2)^{1/(1+\alpha')}$ to balance the $k^{-\alpha'}$ and $n_2 k/m$ terms, and thus obtain
\begin{eqnarray}
\lefteqn{\frac{\E_{\cS,\bY_{\cS}}\left[\|\bX^*-\widehat{\bX}\|_F^2\right]}{n_1 n_2} \leq \frac{76 (\tau \Xmax + 1)^2}{\tau^2} \cdot \frac{\tau \Xmax \log m}{m}}\hspace{4em}&&\\
\nonumber && + \frac{12(\tau\Xmax + 1)^2}{\tau^2} \left(2 + \frac{16\tau\Xmax}{3}\right) (\beta+2)\log (n_1\vee n_2)\left(\frac{n_1 p}{m}\right)\\
\nonumber && +  \frac{12(\tau\Xmax + 1)^2}{\tau^2} \left(\tau \Amax + \left(2 + \frac{16\tau\Xmax}{3}\right) (\beta+2)\log (n_1\vee n_2)\right)\left(\frac{n_2}{m}\right)^{\frac{\alpha'}{\alpha'+1}}.
\end{eqnarray} 

\subsection{Proof of Corollary~\ref{cor:Poi} (Sketch)}\label{a:poiproof}

We follow a similar approach as for the previous proofs, by first establishing a general error bound.  We make use of intermediate results from \cite{Raginsky:10} to bound the KL divergences and negative log Hellinger affinities for the Poisson pmf in terms of quadratic differences. Applying those techniques to our setting, we obtain that
\begin{equation}
\D(p_{X^*_{i,j}}\|p_{X_{i,j}}) \leq \frac{(X^*_{i,j}-X_{i,j})^2}{\Xmin}
\end{equation}
and
\begin{equation}
-2 \log \A(p_{X^*_{i,j}}, p_{X_{i,j}}) \geq \frac{(X^*_{i,j}-X_{i,j})^2}{4\Xmax}.
\end{equation}
It follows that $\D(p_{\bX^*}\|p_{\bX}) \leq \|\bX^*-\bX\|_F^2/\Xmin$, $\E_{\cS, \bY_{\cS}} \left[ -2 \log \A(p_{\widehat{\bX}^{\lambda}},p_{\bX^*}) \right] \geq \E_{\cS, \bY_{\cS}} \left[ \|\bX^*-\widehat{\bX}^{\lambda}\|_F^2 \right]/4\Xmax$, and we may choose $\cD = 4\Xmax^2/\Xmin$.   Incorporating this into Theorem~\ref{thm:main}, we obtain that for any
\begin{equation}
\lambda \geq \left(1 + \frac{8\Xmax^2}{3\Xmin}\right) 2 (\beta + 2) \cdot \log(n_1\vee n_2), 
\end{equation}
the sparsity penalized ML estimate satisfies
\begin{eqnarray}\label{eqn:poigenbnd}
\lefteqn{\frac{\E_{\cS,\bY_{\cS}}\left[\|\bX^*-\widehat{\bX}\|_F^2\right]}{n_1 n_2} \leq  \frac{1}{\Xmin} \cdot \frac{128 \Xmax^3 \log m}{m} + }\hspace{0em} &&\\
\nonumber  && \frac{12 \Xmax}{\Xmin} \cdot \min_{\bX\in\cX} \left\{ \frac{ \|\bX^*-\bX\|_F^2}{n_1 n_2}  + \left(\lambda + \frac{16\Xmax^2 (\beta+2)\log (n_1\vee n_2)}{3}\right)\left(\frac{n_1 r + \|\bA\|_0}{m}\right)\right\}.
\end{eqnarray}
Now, the approximation error term in the oracle bound is in terms of a squared Frobenius norm, so the analysis for the case where $\lambda$ is fixed to the specified value proceeds in an analogous manner to that in Appendix~\ref{a:gaussproof} for both the sparse and approximately sparse settings.  We omit the details.

\subsection{Proof of Corollary~\ref{cor:Bern}}\label{a:bernproof}

For $\bX^*$ as above and any $\bX\in\cX$, and using the model \eqref{eqn:likBern}, it is easy to show that
\begin{equation}
\D(p_{X_{i,j}^*}\|p_{X_{i,j}}) =  F(X_{i,j}^*) \cdot \log \left( \frac{F(X_{i,j}^*)}{F(X_{i,j})} \right) + (1-F(X_{i,j}^*)) \cdot \log \left(\frac{1-F(X_{i,j}^*)}{1-F(X_{i,j})}\right)
\end{equation}
for any fixed $(i,j)\in S$.  Now, we make use of two results that follow directly from lemmata established in \cite{Haupt:14}.  The first lemma provides quadratic bounds on the KL divergence in terms of the Bernoulli parameters; its proof relies on a straightforward application of Taylor's theorem.
\begin{lemmai}[from \cite{Haupt:14}]
Let $p_{\pi}$ and $p_{\pi'}$ be Bernoulli pmf's with parameters $\pi,\pi'\in(0,1)$. The KL divergences satisfy
\begin{eqnarray}
\D(p_{\pi'}\|p_{\pi}), \D(p_{\pi}\|p_{\pi'}) \leq \frac{1}{2} \left(\sup_{|t|\leq \Xmax} \frac{1}{F(t)(1-F(t))}\right) \ (\pi-\pi')^2.
\end{eqnarray}
\end{lemmai}
The second lemma we utilize establishes a bound on the squared difference between Bernoulli parameters in terms of the squared difference of the underlying matrix elements; its proof is straightforward, and essentially entails establishing the Lipschitz continuity of $F$.
\begin{lemmai}[from \cite{Haupt:14}]
Let $\pi=\pi(X)$ and $\pi=\pi'(X')$ be Bernoulli parameters that are related to some underlying real-valued parameters $X$ and $X'$ via $\pi(X) = F(X)$ and $\pi'(X') = F(X')$, where $F(\cdot)$ is the cdf of a continuous random variable with density $f(\cdot)$.  If $|X|, |X'| \leq \Xmax$, then 
\begin{eqnarray}
(\pi(X) - \pi'(X'))^2 &\leq& \left(\sup_{|t|\leq \Xmax} f(t)\right)^2  (X-X')^2,\\
&=& \left(\sup_{|t|\leq \Xmax} f^2(t)\right)  (X-X')^2.
\end{eqnarray} 
\end{lemmai}

Together, these results allow us to claim here that for
\begin{equation}
c_{F,\Xmax} \triangleq \left(\sup_{|t|\leq \Xmax} \frac{1}{F(t)(1-F(t))}\right)\cdot \left(\sup_{|t|\leq \Xmax} f^2(t)\right).
\end{equation}
we have
\begin{equation}
\D(p_{X_{i,j}^*}\|p_{X_{i,j}})  \leq  \frac{1}{2} \cdot c_{F,\Xmax} (X^*_{i,j} - X_{i,j})^2.
\end{equation}
It follows that we may take $\cD = 2 c_{F,\Xmax} \Xmax^2$, and we have $\D(p_{\bX^*}\|p_{\bX})  \leq (c_{F,\Xmax}/2) \  \|\bX^* - \bX\|_F^2$.

We next obtain a (quadratic) lower bound on the negative log Hellinger affinity. To that end, we introduce the squared Hellinger distance between $p_{X_{i,j}^*}$ and $p_{X_{i,j}}$, denoted here by $\H^2(p_{X_{i,j}^*}, p_{X_{i,j}})$ and given by
\begin{equation}
\H^2(p_{X_{i,j}^*}, p_{X_{i,j}}) = \sum_{y\in\{0,1\}} \left(\sqrt{p_{X_{i,j}^*}(y)} - \sqrt{p_{X_{i,j}}(y)}\right)^2.
\end{equation}
It is straightforward to see that $\H^2(p_{X_{i,j}^*}, p_{X_{i,j}}) = 2(1-\A(p_{X_{i,j}^*}, p_{X_{i,j}}))$.  Now, recall that the Hellinger affinity is always between $0$ and $1$, so using the fact that $\log(x)\leq x-1$ for $x>0$, we see directly that
\begin{equation}
\H^2(p_{X_{i,j}^*}, p_{X_{i,j}}) \leq -2\log \A(p_{X_{i,j}^*}, p_{X_{i,j}}).
\end{equation}
Now, a direct application of the result of \cite[Lemma 2]{Davenport:12} derived for a similar subproblem to our problem here yields that for
\begin{equation}
c'_{F,\Xmax} \triangleq \inf_{|t|\leq \Xmax} \frac{f^2(t)}{F(t)(1-F(t))},
\end{equation}
we have that
\begin{equation}
\H^2(p_{X_{i,j}^*}, p_{X_{i,j}}) \geq \frac{1}{8} c'_{F,\Xmax} (X^*_{i,j}-X_{i,j})^2.
\end{equation}
It follows that for any fixed $\bX\in\cX$, we have $-2\log \A(p_{\bX^*}, p_{\bX}) \geq (c'_{F,\Xmax}/8) \ \|\bX^*-\bX\|_F^2$.  

Incorporating all of the above into Theorem~\ref{thm:main} with 
\begin{equation}
\lambda \geq 2 (\beta + 2) \left(1 + \frac{4 c_{F,\Xmax}\  \Xmax^2}{3}\right) \log(n_1\vee n_2), 
\end{equation}
the sparsity penalized ML estimate satisfies the per-element mean-square error bound
\begin{eqnarray}\label{eqn:Bernbnd}
\lefteqn{\frac{\E_{\cS,\bY_{\cS}}\left[\|\bX^*-\widehat{\bX}\|_F^2\right]}{n_1 n_2} \leq  \left(\frac{c_{F,\Xmax}}{c'_{F,\Xmax}}\right) \cdot \frac{128 \Xmax^2 \log m}{m} + }\hspace{0em} &&\\
\nonumber &&\hspace{-1em} 24\left(\frac{c_{F,\Xmax}}{c'_{F,\Xmax}}\right) \cdot \min_{\bX\in\cX} \left\{ \frac{\|\bX^*-\bX\|_F^2}{n_1 n_2}  + \left(\frac{\lambda}{c_{F,\Xmax}} + \frac{8\Xmax^2 (\beta+2)\log (n_1\vee n_2)}{3}\right)\frac{n_1 r + \|\bA\|_0}{m}\right\}.
\end{eqnarray}
Now, the approximation error term in the oracle bound is again in terms of a squared Frobenius norm, so the analysis for the case where $\lambda$ is fixed to the specified value proceeds in an analogous manner to that in Appendix~\ref{a:gaussproof} for both the sparse and nearly sparse settings.  We again omit the details.

\subsection{Proof of Lemma~\ref{lem:main}}\label{a:proof}

Our estimation approach here is, at its essence, a constrained maximum likelihood method and our proof approach follows the general framework proposed in \cite{Li:99:Thesis} (see also \cite{Li:99, Barron:99, Grunwald:07:MDL}) and utilized in \cite{Kolaczyk:04, Willett:07, Raginsky:10}.  Compared with these existing efforts, the main challenge in our analysis here arises because of the ``missing data'' paradigm, since we aim to establish consistency results that hold \emph{globally} (at all locations of the unknown matrix) using observations obtained at only a subset of the locations.  Our approach will be to identify conditions under which, for the purposes of our analysis, a set of sample locations is deemed ``good,'' in a manner to be made explicit below.  The primary characteristic of good sets $S$ of sample locations that we will leverage in our analysis is that they be such that KL divergences and (negative logarithms of) Hellinger affinities evaluated only at the locations in $S$ be representative surrogates for the corresponding quantities were we to evaluate them at all $(i,j)\in [n_1]\times [n_2]$ (i.e., even at the unmeasured locations).  Clearly, such conditions will inherently rely on certain properties of the matrices that we seek to estimate, somewhat analogously to how notions of incoherence facilitate matrix completion analyses under low rank matrix models.  Here, we will see these conditions manifest not as properties of the singular vectors of the unknown matrix to be estimated as in existing matrix completion works, but instead, as conditions on the magnitude of the largest matrix entry.

Our approach will be as follows.  First, we describe formally the notion of ``good'' sets of sample locations, and we show that sets of sample locations generated randomly according to an independent Bernoulli model are ``good'' with high probability.  Then, we establish error guarantees that hold conditionally on the event that the set of sample locations is ``good.''  Finally, we obtain our overall result using some simple conditioning arguments.

\subsubsection{``Good'' Sample Set Characteristics}

We begin by characterizing, formally, the properties of certain sets of sample locations that will be useful for our analysis here.  As above $\bX^*$ denotes the true (unknown) matrix that we aim to estimate, and $\cX$ is a countable set of candidate estimates $\bX$, each with corresponding penalty $\pen(\bX)\geq 1$ chosen so the inequality \eqref{eqn:Kraft} is satisfied.  Also, recall that $\Xmax>0$ is a finite constant for which $\max_{i,j} |X^*_{i,j}| \leq \Xmax/2$ and $\max_{\bX\in\cX} \max_{i,j} |X_{i,j}| \leq \Xmax$. Finally, we let $\cA$ and $\cD$ be any upper bounds, respectively, on (twice) the negative log Hellinger affinities between $p_{X^*_{i,j}}$ and $p_{X_{i,j}}$, and the KL divergences of $p_{X_{i,j}}$ from $p_{X^*_{i,j}}$ that hold over all indices, and for all elements $\bX\in\cX$, so that
\begin{equation}
\cA \geq \max_{\bX\in\cX} \max_{i,j}  \ -2\log \A(p_{X^*_{i,j}},p_{X_{i,j}})
\end{equation}
and
\begin{equation}
\cD \geq \max_{\bX\in\cX} \max_{i,j}  \ D(p_{X^*_{i,j}}\|p_{X_{i,j}}).
\end{equation} 
Note that the statement of Theorem~\ref{thm:main} only prescribed a condition on $\cD$; our introduction of an additional constant $\cA$ here is only to simplify the subsequent analysis.  In the concluding steps of the proof we will claim that upon selecting a suitable $\cD$, one may always obtain a valid choice of $\cA$ by taking $\cA=\cD$. This will enable us to eliminate the $\cA$ terms that arise in our bound by bounding them in terms of the constant $\cD$.

Let $m\in[n_1 n_2]$ denote a nominal number of measurements, and let $\p=m/n_1n_2\in(0,1]$ denote the corresponding nominal fraction of observed matrix elements.  For this $\p$ and any fixed $\delta\in(0,1)$, we define the ``good'' set ${\cal G}_{\p,\delta} = {\cal G}_{\p,\delta}(\bX^*, \cX)$ of possible sample location sets as
\begin{eqnarray}
\lefteqn{{\cal G}_{\p,\delta} \triangleq \Bigg\{S \subseteq [n_1]\times [n_2] \ : 
 \bigcap_{\bX\in\cX} D(p_{\bX_{S}^*}\|p_{\bX_{S}}) \leq \frac{3\p}{2}\D(p_{\bX^*}\|p_{\bX}) + 2\left(\frac{2 \cD}{3}\right)\left[\log(1/\delta) + \pen(\bX)\log 2\right] }\hspace{2em}&&\\
\nonumber && \ \cap \ \bigcap_{\bX\in\cX} (-2\log \A(p_{\bX_{S}^*},p_{\bX_{S}})) \geq \frac{\p}{2}\left(-2\log \A(p_{\bX^*},p_{\bX})\right) - 2\left(\frac{2 \cA}{3}\right)\left[\log(1/\delta) + \pen(\bX)\log 2\right]\Bigg\}.
\end{eqnarray}
Directly certifying whether any fixed set $S$ is an element of ${\cal G}_{\p,\delta}$ may be difficult in general.  However, our observation model here assumes that the sample location set is generated randomly, according to an independent Bernoulli($\p$) model, where each location is included in the set independently with probability $\p\in(0,1]$.  In this case, we have that random sample location sets $\cS$ so generated satisfy $\cS\in{\cal G}_{\p,\delta}$ with high probability, as shown in the following lemma.
\begin{lemmai}
Let $\cX$ be any countable collection of candidate estimates $\bX$ for $\bX^*$, with corresponding penalties $\pen(\bX)$ satisfying \eqref{eqn:Kraft}.  For any fixed $\p\in(0,1)$, let $\cS\subseteq [n_1]\times [n_2]$ be a random sample set generated according to the independent Bernoulli($\p$) model.  Then, for any $\delta\in(0,1)$ we have $\Pr(\cS \notin {\cal G}_{\p,\delta}) \leq 2\delta$.
\end{lemmai}

\begin{proof}
Write $\{\cS\in{\cal G}_{\kappa,\delta}\} = {\cal E}_u \cap {\cal E}_l$, where 
\begin{equation}
{\cal E}_u \triangleq \left\{ \bigcap_{\bX\in\cX} D(p_{\bX_{S}^*}\|p_{\bX_{S}}) \leq \frac{3\p}{2}\D(p_{\bX^*}\|p_{\bX}) + 2\left(\frac{2 \cD}{3}\right)\left[\log(1/\delta) + \pen(\bX)\log 2\right]\right\},
\end{equation}
and 
\begin{equation}
{\cal E}_l \triangleq \left\{ \bigcap_{\bX\in\cX} (-2\log \A(p_{\bX_{S}^*},p_{\bX_{S}})) \geq \frac{\p}{2}\left(-2\log \A(p_{\bX^*},p_{\bX})\right) - 2\left(\frac{2 \cA}{3}\right)\left[\log(1/\delta) + \pen(\bX)\log 2\right]\right\},
\end{equation}
Then, by straight-forward union bounding, $\Pr(\cS\notin{\cal G}_{\p,\delta}(\cX))\leq \Pr({\cal E}_u^c) + \Pr({\cal E}_l^c)$. The proof of the lemma entails bounding each term on the right-hand side, in turn.

We focus first on bounding the probability of the complement of ${\cal E}_u$. To proceed, we will find it convenient to consider an alternative (but equivalent) representation of the sampling operator described explicitly in terms of a collection $\{B_{i,j}\}_{(i,j)}\in[n_1]\times[n_2]$ of independent Bernoulli($\p$) random variables, so that $\cS = \{(i,j) : B_{i,j} = 1\}$.  On account of our assumption that the observations be conditionally independent given $\cS$, we have that for any fixed $\bX\in\cX$,
\begin{equation}
D(p_{\bX_{\cS}^*}\|p_{\bX_{\cS}}) = \sum_{(i,j)\in\cS} \D(p_{X^*_{i,j}},p_{X_{i,j}}) = \sum_{i,j} B_{i,j} \cdot \D(p_{X^*_{i,j}},p_{X_{i,j}}).
\end{equation}
Thus, our analysis reduces to quantifying the concentration behavior of random sums of these forms. For this, we employ a powerful version of Bernstein's Inequality established by Craig \cite{Craig:33} that, for our purposes, may be stated as follows: let $\{U_{i,j}\}$ be a collection of independent random variables indexed by $(i,j)$, each satisfying the moment condition that for some $h>0$,
\begin{equation*}
\E\left[|U_{i,j} - \E[U_{i,j}]|^k \right] \leq \frac{\mbox{var}(U_{i,j})}{2} \ k! \ h^{k-2}, 
\end{equation*} 
for all integers $k\geq 2$.  Then, for any $\tau>0$ and $0 \leq \epsilon h \leq \theta < 1$, the probability that
\begin{equation}\label{eqn:cb}
\sum_{i,j} (U_{i,j} - \E\left[U_{i,j}\right]) \geq \frac{\tau}{\epsilon} + \frac{\epsilon \ \sum_{i,j} \mbox{var}\left(U_{i,j} \right)}{2(1-\theta)} 
\end{equation} 
is no larger than $e^{-\tau}$.  A useful (and easy to verify) fact is that whenever $|U_{i,j} - \E[U_{i,j}]|\leq \beta$, the moment condition is satisfied by the choice $h=\beta/3$.  

Now, fix $\bX\in\cX$, and let 
$U_{i,j}(\bX) = B_{i,j} \cdot \D(p_{X^*_{i,j}},p_{X_{i,j}})$ and $\E\left[U_{i,j}(\bX)\right] = \p \cdot \D(p_{X^*_{i,j}},p_{X_{i,j}})$. Applying Craig's version of Bernstein's inequality with $\theta = 1/4$, $h=\cD/3$, and $\epsilon = \theta/h = 3/(4\cD)$, and using the fact that
\begin{equation}
\mbox{var}(U_{i,j}(\bX)) = \p(1-\p) \left(\D(p_{X^*_{i,j}},p_{X_{i,j}}) \right)^2 \leq \p\left(\D(p_{X^*_{i,j}},p_{X_{i,j}})\right)^2
\end{equation}
we obtain that for any $\tau>0$,
\begin{equation}
\Pr\left(\sum_{i,j} (B_{i,j}-\p) \D(p_{X^*_{i,j}},p_{X_{i,j}}) \geq 
\frac{4\cD \tau}{3} + \frac{\sum_{i,j} \p \left(\D(p_{X^*_{i,j}},p_{X_{i,j}})\right)^2}{2\cD} \right)\leq e^{-\tau}.
\end{equation}

Now, since $\D(p_{X^*_{i,j}},p_{X_{i,j}}) \leq \cD$ by definition, the above result ensures that for any $\tau>0$, 
\begin{equation}
\Pr\left(\sum_{i,j} (B_{i,j}-\p) \D(p_{X^*_{i,j}},p_{X_{i,j}}) \geq 
\frac{4\cD \tau}{3} + \frac{\p}{2} \D(p_{\bX^*},p_{\bX}) \right)\leq e^{-\tau}.
\end{equation}
Letting $\delta = e^{-\tau}$ and simplifying a bit, we obtain that for any $\delta\in(0,1)$,
\begin{equation}
\Pr\left(\D(p_{\bX^*_{\cS}},p_{\bX_{\cS}}) \geq 
\frac{4\cD \log(1/\delta)}{3} + \frac{3\p}{2} \D(p_{\bX^*},p_{\bX}) \right) \leq \delta.
\end{equation}
Now, if for each $\bX\in\cX$ we let $\delta_{\bX} = \delta\cdot 2^{-\pen(\bX)}$, we can apply the union bound to obtain that 
\begin{equation}\label{eqn:Dpiece}
\Pr\left(\bigcup_{\bX\in\cX} \D(p_{\bX^*_{\cS}},p_{\bX_{\cS}}) \geq 
\frac{3\p}{2} \D(p_{\bX^*},p_{\bX}) + 2\left(\frac{2 \cD}{3}\right) \left[\log(1/\delta) + \pen(\bX)\cdot \log 2\right] \right) \leq \delta.
\end{equation}
Following a similar approach for the affinity terms (with $U_{i,j}(\bX) = -B_{i,j} \cdot (-2\log \A(p_{X^*_{i,j}},p_{X_{i,j}}))$ for all $i,j$), we obtain that for any $\delta\in(0,1)$, 
\begin{equation}\label{eqn:Apiece}
\Pr\left(\bigcup_{\bX\in\cX} \left(-2\log\A(p_{\bX^*_{\cS}},p_{\bX_{\cS}})\right) \leq 
\frac{\p}{2}\left(-2\log \A(p_{\bX^*},p_{\bX})\right) - 2\left(\frac{2 \cA}{3}\right) \left[\log(1/\delta) + \pen(\bX)\cdot \log 2\right] \right) \leq \delta.
\end{equation}
The overall result now follows by combining equations \eqref{eqn:Dpiece} and \eqref{eqn:Apiece} using a union bound.
\end{proof}
Next, we show how the implications of a sample set being ``good'' can be incorporated into the analysis of \cite{Li:99:Thesis} to provide (conditional) error guarantees for completion tasks.

\subsubsection{A Conditional Error Guarantee}

Next, we establish the consistency of complexity penalized maximum likelihood estimators, conditionally on the event that the sample set $\cS$ is a fixed set $S$, such that for fixed $\p\in(0,1)$ and $\delta\in(0,1)$, $S\in{\cal G}_{\p,\delta}$ (i.e., $S$ is ``good'' according to the criteria outlined above). Our analysis then proceeds along the lines of the approach of \cite{Li:99:Thesis}, but with several key differences that arise because of our subsampling model.

As above, $\cX$ is a countable set of candidate estimates $\bX$ for $\bX^*$, with corresponding penalties $\pen(\bX)$ satisfying \eqref{eqn:Kraft}. Now, for any choice of $\mu$ satisfying $\mu \geq 1 + 2\cA/3$, we form an estimate $\widehat{\bX}^{\mu} = \widehat{\bX}^{\mu}(\bY_{S})$ according to
\begin{eqnarray}
\nonumber \widehat{\bX}^{\mu} &=& \arg \min_{\bX\in\cX} \left\{-\log p_{\bX_{S}}(\bY_{S}) + 2\mu \cdot \pen(\bX) \log 2\right\}\\
&=& \arg \max_{\bX\in\cX} \left\{\sqrt{p_{\bX_{S}}(\bY_{S})} \cdot 2^{-\mu \cdot \pen(\bX)}\right\}.
\end{eqnarray}
By this choice, we have that for any $\bX\in\cX$,
\begin{equation}
\sqrt{p_{\widehat{\bX}^{\mu}_{S}}(\bY_{S})} \cdot 2^{-\mu \cdot \pen(\widehat{\bX}^{\mu})} \geq \sqrt{p_{\bX_{S}}(\bY_{S})} \cdot 2^{-\mu \cdot \pen(\bX)}.
\end{equation}
This implies that for the particular (deterministic, and $\mu$-dependent) candidate
\begin{equation}
\widetilde{\bX}^{\mu} = \arg \min_{\bX\in\cX}  \left\{ \D(p_{\bX^*}\|p_{\bX}) + \frac{2}{\p}\cdot \left(\mu + \frac{2\cD}{3}\right) \pen(\bX) \log 2 \right\},
\end{equation}
(whose specification will become clear shortly) we have
\begin{equation}
\frac{\sqrt{p_{\widehat{\bX}^{\mu}_{S}}(\bY_{S})} \cdot 2^{-\mu \cdot \pen(\widehat{\bX}^{\mu})}}{\sqrt{p_{\widetilde{\bX}^{\mu}_{S}}(\bY_{S})} \cdot 2^{-\mu \cdot \pen(\widetilde{\bX}^{\mu})}} \geq 1.
\end{equation}
Using this, along with some straight-forward algebraic manipulations, we have
\begin{eqnarray}\label{eqn:step1}
\nonumber \lefteqn{-2\log \A(p_{\widehat{\bX}^{\mu}_{S}}, p_{\bX^*_{S}}) = 2\log \left(\frac{1}{\A(p_{\widehat{\bX}^{\mu}_{S}}, p_{\bX^*_{S}})}\right)}\hspace{3em}&&\\
\nonumber  &\leq& 2\log \left(\frac{\sqrt{p_{\widehat{\bX}^{\mu}_{S}}(\bY_{S})} \cdot 2^{-\mu \cdot \pen(\widehat{\bX}^{\mu})}}{\sqrt{p_{\widetilde{\bX}^{\mu}_{S}}(\bY_{S})} \cdot 2^{-\mu \cdot \pen(\widetilde{\bX}^{\mu})}} \cdot \frac{1}{\A(p_{\widehat{\bX}^{\mu}_{S}}, p_{\bX^*_{S}})}\right)\\
\nonumber  &=& 2\log \left(\frac{\sqrt{p_{\widehat{\bX}^{\mu}_{S}}(\bY_{S})}}{\sqrt{p_{\widetilde{\bX}^{\mu}_{S}}(\bY_{S})}}\cdot \frac{\sqrt{p_{\bX^*_{S}}(\bY_{S})}}{\sqrt{p_{\bX^*_{S}}(\bY_{S})}}\cdot \frac{2^{-\mu \cdot \pen(\widehat{\bX}^{\mu})}}{2^{-\mu \cdot \pen(\widetilde{\bX}^{\mu})}} \cdot \frac{1}{\A(p_{\widehat{\bX}^{\mu}_{S}}, p_{\bX^*_{S}})}\right)\\
&=& \log\left(\frac{p_{\bX^*_{S}}(\bY_{S})}{p_{\widetilde{\bX}^{\mu}_{S}}(\bY_{S})}\right) + 2\mu \cdot \pen(\widetilde{\bX}^{\mu})\log 2 + 
2 \log\left(
\frac{\sqrt{p_{\widehat{\bX}^{\mu}_{S}}(\bY_{S})/p_{\bX^*_{S}}(\bY_{S})}}
{\A(p_{\widehat{\bX}^{\mu}_{S}}, p_{\bX^*_{S}})} \cdot 2^{-\mu \cdot \pen(\widehat{\bX}^{\mu})}\right).
\end{eqnarray}

At this point, we make our first use of the implications of the ``good'' sample set condition.  In particular, since $S\in{\cal G}_{\p,\delta}$ and $\widehat{\bX}^{\mu}\in\cX$, we have that
\begin{equation}
-2\log \A(p_{\bX_{S}^*},p_{\widehat{\bX}^{\mu}_{S}}) \geq \frac{\p}{2}\left(-2\log \A(p_{\bX^*},p_{\widehat{\bX}^{\mu}})\right) - 2\left(\frac{2 \cA}{3}\right)\left[\log(1/\delta) + \pen(\widehat{\bX}^{\mu})\log 2\right].
\end{equation}
Incorporating this into \eqref{eqn:step1}, we have
\begin{eqnarray}
\nonumber \lefteqn{\frac{\p}{2} \left(-2\log \A(p_{\widehat{\bX}^{\mu}}, p_{\bX^*})\right) \leq  \log\left(\frac{p_{\bX^*_{\cS}}(\bY_{S})}{p_{\widetilde{\bX}^{\mu}_{S}}(\bY_{S})}\right) + 2\mu \cdot \pen(\widetilde{\bX}^{\mu})\log 2 + 2\left(\frac{2 \cA}{3}\right) \log(1/\delta) }\hspace{10em}&&\\
&&  +2 \log\left(
\frac{\sqrt{p_{\widehat{\bX}^{\mu}_{S}}(\bY_{S})/p_{\bX^*_{S}}(\bY_{S})}}
{\A(p_{\widehat{\bX}^{\mu}_{S}}, p_{\bX^*_{S}})} \cdot 2^{-\left(\mu- \frac{2 \cA}{3}\right) \pen(\widehat{\bX}^{\mu})}\right).
\end{eqnarray}
Now, we take expectations (formally, with respect to the conditional distribution of $\bY_{\cS}$ given $\{\cS=S, S\in{\cal G}_{\p,\delta}\}$) on both sides to obtain that
\begin{eqnarray}
\nonumber \lefteqn{\frac{\p}{2} \E\left[-2\log \A(p_{\widehat{\bX}^{\mu}}, p_{\bX^*}) \ \bigg| \ \cS = S, S\in{\cal G}_{\p,\delta}\right] \leq  \D(p_{\bX^*_{S}}\|p_{\widetilde{\bX}^{\mu}_{S}}) + 2\mu \cdot \pen(\widetilde{\bX}^{\mu})\log 2 + 2\left(\frac{2 \cA}{3}\right) \log(1/\delta)}\hspace{7em}&&\\
&& + 2 \E\left[\log\left(
\frac{\sqrt{p_{\widehat{\bX}^{\mu}_{S}}(\bY_{S})/p_{\bX^*_{S}}(\bY_{S})}}
{\A(p_{\widehat{\bX}^{\mu}_{S}}, p_{\bX^*_{S}})} \cdot 2^{-\left(\mu - \frac{2 \cA}{3}\right) \pen(\widehat{\bX}^{\mu})}\right) \ \bigg| \ \cS=S, S\in{\cal G}_{\p,\delta}\right].
\end{eqnarray}
Using again the implications of $S\in{\cal G}_{\p,\delta}$, that 
\begin{equation}
D(p_{\bX_{S}^*}\|p_{\widetilde{\bX}^{\mu}_{S}}) \leq \frac{3\p}{2}\D(p_{\bX^*}\|p_{\widetilde{\bX}^{\mu}}) + 2\left(\frac{2 \cD}{3}\right)\left[\log(1/\delta) + \pen(\widetilde{\bX}^{\mu})\log 2\right]
\end{equation}
since $\widetilde{\bX}^{\mu}\in\cX$, we have that 
\begin{eqnarray}\label{eqn:withlastterm}
\nonumber \lefteqn{\E\left[-2\log \A(p_{\widehat{\bX}^{\mu}}, p_{\bX^*}) \ \bigg| \ \cS=S, S\in{\cal G}_{\p,\delta}\right] \leq}\hspace{5em}&&\\
\nonumber && 3 \D(p_{\bX^*}\|p_{\widetilde{\bX}^{\mu}}) + \frac{4}{\p}\left(\mu  + \frac{2\cD}{3}\right)\pen(\widetilde{\bX}^{\mu})\log 2 + \frac{4}{\p}\left(\frac{2 (\cA+ \cD)}{3}\right) \log(1/\delta) \\
&& + \frac{4}{\p} \E\left[\log\left(\frac{\sqrt{p_{\widehat{\bX}^{\mu}_{S}}(\bY_{S})}/\sqrt{p_{\bX^*_{S}}(\bY_{S})}}{\A(p_{\widehat{\bX}^{\mu}_{S}}, p_{\bX^*_{S}})} \cdot 2^{-\left(\mu - \frac{2 \cA}{3}\right) \pen(\widehat{\bX}^{\mu})}\right) \ \bigg| \ \cS=S, S\in{\cal G}_{\p,\delta}\right].
\end{eqnarray}

Turning our attention to the last term on the right-hand side, we have that
\begin{eqnarray}
\nonumber \lefteqn{\E\left[\log\left(\frac{\sqrt{p_{\widehat{\bX}^{\mu}_{S}}(\bY_{S})}/\sqrt{p_{\bX^*_{S}}(\bY_{S})}}{\A(p_{\widehat{\bX}^{\mu}_{S}}, p_{\bX^*_{S}})} \cdot 2^{-\left(\mu - \frac{2 \cA}{3}\right) \pen(\widehat{\bX}^{\mu})}\right) \ \bigg| \ \cS=S, S\in{\cal G}_{\p,\delta}\right]}\hspace{8em}&&\\
\nonumber && \stackrel{(a)}{\leq} \log\left( \E\left[\frac{\sqrt{p_{\widehat{\bX}^{\mu}_{S}}(\bY_{S})}/\sqrt{p_{\bX^*_{S}}(\bY_{S})}}{\A(p_{\widehat{\bX}^{\mu}_{S}}, p_{\bX^*_{S}})} \cdot 2^{-\left(\mu - \frac{2 \cA}{3}\right) \pen(\widehat{\bX}^{\mu})} \ \bigg| \ \cS=S, S\in{\cal G}_{\p,\delta}\right]\right)\\
\nonumber && \stackrel{(b)}{\leq} \log\left( \E\left[ \sum_{\bX\in\cX} \frac{\sqrt{p_{\bX_{S}}(\bY_{S})}/\sqrt{p_{\bX^*_{S}}(\bY_{S})}}{\A(p_{\bX_{S}}, p_{\bX^*_{S}})} \cdot 2^{-\left(\mu - \frac{2 \cA}{3}\right) \pen(\bX)} \ \bigg| \ \cS=S, S\in{\cal G}_{\p,\delta}\right]\right)\\
\nonumber && = \log\left(  \sum_{\bX\in\cX} 2^{-\left(\mu - \frac{2 \cA}{3}\right) \pen(\bX)} \ \E\left[ \frac{\sqrt{p_{\bX_{S}}(\bY_{S})}/\sqrt{p_{\bX^*_{S}}(\bY_{S})}}{\A(p_{\bX_{S}}, p_{\bX^*_{S}})} \ \bigg| \ \cS=S, S\in{\cal G}_{\p,\delta}\right]\right)\\
&& \stackrel{(c)}{=} \log\left(  \sum_{\bX\in\cX} 2^{-\left(\mu - \frac{2 \cA}{3}\right) \pen(\bX)}\right).
\end{eqnarray}
In the above, $(a)$ follows from Jensen's Inequality, $(b)$ from the facts that $\widehat{\bX}^{\mu}\in\cX$ and each term in the sum is non-negative, and $(c)$ from the definition of the Hellinger affinity.  Now, because $\pen(\bX) \geq 1$ and $\mu \geq 1+2\cA/3$ we have that 
\begin{equation}
\sum_{\bX\in\cX} 2^{-\left(\mu - \frac{2 \cA}{3}\right) \pen(\bX)} \leq \sum_{\bX\in\cX} 2^{-\pen(\bX)} \leq 1.
\end{equation}
Thus, since the expectation term on the right-hand side of \eqref{eqn:withlastterm} is not positive, we can disregard it in the upper bound to obtain that
\begin{eqnarray}
\nonumber \lefteqn{\E\left[-2\log \A(p_{\widehat{\bX}^{\mu}}, p_{\bX^*}) \ \bigg| \ \cS=S, S\in{\cal G}_{p,\delta}\right]}\hspace{6em}&&\\
&\leq& 3 \D(p_{\bX^*}\|p_{\widetilde{\bX}^{\mu}}) + \frac{6}{\p}\left(\lambda + \frac{2\cD}{3}\right)\pen(\widetilde{\bX}^{\mu})\log 2 + \frac{4}{\p}\left(\frac{2 (\cA+ \cD)}{3}\right) \log(1/\delta),
\end{eqnarray}
where we have also inflated (slightly) the leading constant on the second term on the right-hand side to simplify subsequent analysis.   Now, recalling the definition of $\widetilde{\bX}^{\mu}$, we can state the result equivalently as an oracle bound, as
\begin{eqnarray}\label{eqn:condresult}
\nonumber \lefteqn{\E\left[-2\log \A(p_{\widehat{\bX}^{\mu}}, p_{\bX^*}) \ \bigg| \ \cS=S, S\in{\cal G}_{\p,\delta}\right]}\hspace{2em}&&\\
&\leq&  3 \cdot \min_{\bX\in\cX} \left\{ \D(p_{\bX^*}\|p_{\bX}) + \frac{2}{\p}\left(\mu + \frac{2\cD}{3}\right)\pen(\bX)\log 2\right\} 
+ \frac{4}{\p}\left(\frac{2 (\cA+ \cD)}{3}\right) \log(1/\delta) .
\end{eqnarray}

\subsubsection{Putting the Pieces Together}

The last steps of the analysis entail straightforward applications of conditioning arguments, along with the use of a well-known (and easy to verify) information inequality.  First, note that
\begin{eqnarray}
\nonumber \lefteqn{\E\left[-2\log \A(p_{\widehat{\bX}^{\mu}}, p_{\bX^*}) \ \bigg| \ \cS\in{\cal G}_{\p,\delta}\right]}\hspace{2em}&&\\
\nonumber &=& \sum_{S\in[n_1]\times [n_2]} \E\left[-2\log \A(p_{\widehat{\bX}^{\mu}}, p_{\bX^*}) \ \bigg| \ \cS=S, S\in{\cal G}_{\p,\delta}\right]\cdot \Pr(\cS=S | \cS\in{\cal G}_{\p,\delta})\\
&\leq& 3 \cdot \min_{\bX\in\cX} \left\{ \D(p_{\bX^*}\|p_{\bX}) + \frac{2}{\p}\left(\mu + \frac{2\cD}{3}\right)\pen(\bX)\log 2\right\} 
+ \frac{4}{\p}\left(\frac{2 (\cA+ \cD)}{3}\right) \log(1/\delta),
\end{eqnarray}
where the last step follows from using the bound in \eqref{eqn:condresult} and bringing that term outside of the sum since it does not depend on $S$, and using the fact that the conditional probability mass function $\Pr(\cS=S | \cS\in{\cal G})$ sums to $1$.  Now, using the fact that
\begin{eqnarray}
\lefteqn{\E\left[-2\log \A(p_{\widehat{\bX}^{\mu}}, p_{\bX^*})\right] = }&&\\
\nonumber && \E\left[-2\log \A(p_{\widehat{\bX}^{\mu}}, p_{\bX^*}) \ \bigg| \ \cS\in{\cal G}_{\p,\delta}\right] \cdot \Pr(\cS\in{\cal G}_{\p,\delta}) + \E\left[-2\log \A(p_{\widehat{\bX}^{\mu}}, p_{\bX^*}) \ \bigg| \ \cS\notin{\cal G}_{\p,\delta}\right] \cdot \Pr(\cS\notin{\cal G}_{\p,\delta}),
\end{eqnarray}
where the expectation on the left-hand side is with respect to the joint distribution of $\bY_{\cS}$ and $\cS$, we obtain that
\begin{eqnarray}
\lefteqn{\E\left[-2\log \A(p_{\widehat{\bX}^{\mu}}, p_{\bX^*})\right] \leq }&&\\
\nonumber && 3 \cdot \min_{\bX\in\cX} \left\{ \D(p_{\bX^*}\|p_{\bX}) + \frac{2}{\p}\left(\mu + \frac{2\cD}{3}\right)\pen(\bX)\log 2\right\} + \frac{4}{\p}\left(\frac{2 (\cA+ \cD)}{3}\right) \log(1/\delta)  + 2\delta \cdot n_1 n_2 \cA,
\end{eqnarray}
where we use the trivial upper bound $\E [-2\log \A(p_{\widehat{\bX}^{\mu}}, p_{\bX^*}) \ | \ \cS\notin{\cal G}_{\p,\delta}] \leq n_1 n_2 \cA$.  Now, since the result holds for any choice of $\delta\in(0,1)$, we can choose $\delta$ judiciously to ``balance'' the last two terms.  The particular choice $\delta = m^{-1} = (\p n_1n_2)^{-1}$ yields
\begin{eqnarray}
\nonumber \lefteqn{\E\left[-2\log \A(p_{\widehat{\bX}^{\mu}}, p_{\bX^*})\right]}&&\\
&\leq& 3 \cdot \min_{\bX\in\cX} \left\{ \D(p_{\bX^*}\|p_{\bX}) + \frac{2}{\p}\left(\mu + \frac{2\cD}{3}\right)\pen(\bX)\log 2\right\} + 
\frac{8 (\cA+ \cD) \log m}{3\p} + \frac{2\cA}{\p},
\end{eqnarray}
which implies the simpler (but slightly looser) bound
\begin{equation}
\E\left[-2\log \A(p_{\widehat{\bX}^{\mu}}, p_{\bX^*})\right] \leq 3 \cdot \min_{\bX\in\cX} \left\{ \D(p_{\bX^*}\|p_{\bX}) + \frac{2}{\p}\left(\mu + \frac{2\cD}{3}\right)\pen(\bX)\log 2\right\} + \frac{4\left(\cA+ \cD\right) \log m}{\p}.
\end{equation}

Finally, we make use of the fact that for each $i,j$, we have $-2\log \A(p_{X^*_{i,j}},p_{X_{i,j}}) \leq \D(p_{X^*_{i,j}},p_{X_{i,j}})$, which is readily verified with one application of Jensen's inequality.  It follows that upon identifying a suitable $\cD$, we may always take $\cA=\cD$.  Thus, it is sufficient to choose $\mu > 1+2\cD/3$ when forming our complexity regularized maximum likelihood estimator. We conclude that the error of any estimator formed using an appropriate regularization parameter $\mu$ satisfies
\begin{equation}
\frac{\E\left[-2\log \A(p_{\widehat{\bX}^{\mu}}, p_{\bX^*})\right]}{n_1 n_2} \leq 3 \cdot \min_{\bX\in\cX} \left\{ \frac{\D(p_{\bX^*}\|p_{\bX})}{n_1n_2} + \left(\mu + \frac{2\cD}{3}\right)\frac{\pen(\bX)2\log 2}{m}\right\} + \frac{8\cD \log(m)}{m},
\end{equation}
where we have divided both sides by $n_1 n_2$ and used the fact that $m=\p n_1 n_2$. Finally, making the substitution $\xi = 2\mu \log2$ yields the stated version of the result.

\bibliographystyle{IEEEbib}
\bibliography{NMCbib}

\end{document}